\documentclass[twoside]{article}

\usepackage{aistats2027}
\usepackage[utf8]{inputenc}
\usepackage[T1]{fontenc}    
\usepackage{hyperref}       
\usepackage{url}            
\usepackage{booktabs}       
\usepackage{amsfonts}       
\usepackage{nicefrac}       
\usepackage{microtype}      
\usepackage{xcolor}         
\usepackage{amsmath}
\usepackage{amssymb}
\usepackage{amsthm}
\usepackage{graphicx}
\usepackage{multirow}
\usepackage{multicol}
\usepackage{natbib}
\usepackage{tcolorbox}
\usepackage{listings}

\newtheorem{theorem}{Theorem}
\newtheorem{lemma}{Lemma}

\newtheorem{proposition}{Proposition}

\begin{document}

\twocolumn[
\runningauthor{Evaluating and Recalibrating Probabilistic Classifiers from One-Bit Threshold Queries}
\runningtitle{Evaluating and Recalibrating Probabilistic Classifiers from One-Bit Threshold Queries}
\aistatstitle{Evaluating and Recalibrating Probabilistic Classifiers \\from One-Bit Threshold Queries}
\aistatsauthor{
Roman Plaud \\
Institut Polytechnique de Paris \\
Onepoint, France
\And
Antoine Saillenfest \\
Onepoint, France
\And
Matthieu Labeau \\
Institut Polytechnique de Paris
\AND
Thomas Bonald \\
Institut Polytechnique de Paris
\And
Willem Waegeman \\
Ghent University
}

\aistatsaddress{}
]

\begin{abstract}
Standard calibration evaluation and recalibration methods assume access to predicted probabilities. We study \emph{one-bit threshold access}, a setting in which probabilities are hidden and the evaluator observes only whether the predicted probability exceeds a single chosen threshold per sample. Under this restriction, we show that classifiers can nevertheless be evaluated and recalibrated without reconstructing their hidden scores. We derive a consistent estimator of True Calibration Error, unbiased estimators of
strictly proper scoring rules, and a one-bit objective for Temperature Scaling
whose expectation matches the standard log-loss objective. Across these settings, we quantify the statistical cost incurred by replacing standard probability-access procedures by our one-bit estimators. Finally, we instantiate the framework for black-box LLMs, where the \texttt{logit\_bias} API parameter allows us to realize the required threshold queries in zero-shot binary question answering.
\end{abstract}

\section{Introduction}

Calibration is a desirable property of probabilistic classifiers, capturing the extent to which predicted probabilities align with the empirical frequencies of the events they describe \citep{guo_calib}. Standard calibration procedures, which span from estimating calibration
error \citep{naeini2015obtaining,popordanoska2022consistent} to evaluating
proper scoring rules \citep{gneiting2007strictly}, to fitting post-hoc
recalibration maps
\citep{platt1999probabilistic,kull2017beta,berta26a}, share a
basic assumption: access to the model's predicted scores.

This assumption does not necessarily hold when predictors are deployed through
restricted black-box interfaces. A service may output only discrete objects (\textit{e.g.} labels or tokens) while keeping the underlying predicted scores hidden. In this paper, we study
an intermediate observation model between full probability access and purely
label-only prediction. Specifically, given a dataset
$(X_i,Y_i)_{i=1}^N$ with binary labels, we assume that the predicted score
$S_i=f(X_i)\in(0,1)$ is hidden, but that we may choose a
threshold $t\in(0,1)$ and observe the binary response $Q_t(X_i)=\mathbf{1}\{S_i\ge t\}.$ We refer to this observation model as \emph{one-bit threshold access}.

With sufficiently many threshold queries on the same observation, the hidden
probability could be reconstructed to arbitrary precision by binary search~\citep{carlini2024stealing},
after which standard probability-access procedures could be applied. However, this turns evaluation into a repeated-query
procedure, potentially requiring many model calls for each example.

We instead ask what can be done at the cost of ordinary evaluation: at most
\emph{one threshold query per observation}. This keeps the number of model
calls linear in the number of evaluation examples and rules out reconstructing
each hidden probability before using it. Each continuous score $S_i$ is then
replaced by a single comparison bit. This leads to a different statistical problem. Rather than
first recovering the probabilities and then evaluating them, we ask whether
calibration quantities can be estimated \emph{directly from threshold
observations}.

We show that one-bit threshold access is sufficient to support both calibration evaluation and post-hoc recalibration without reconstructing individual scores. We develop one-bit procedures spanning calibration-error estimation, proper-scoring-rule evaluation, and Temperature Scaling, and characterize the statistical cost incurred by replacing the standard probability-access procedures by our one-bit estimators.

A motivating application is the use of large language models (LLMs) as
zero-shot binary classifiers under restricted interfaces. When token
probabilities are hidden, the evaluator may observe only the model's discrete
output. We show that the \texttt{logit\_bias} parameter, exposed by some LLM
APIs \citep{openai_logit_bias}, can be used to realize the
required threshold queries. Hence, our framework provides a principled approach to evaluate the calibration and recalibrate black-box LLMs via one-bit threshold access.

Our main contributions are as follows:
\begin{itemize}
    \item We formalize \emph{one-bit threshold access}
    for probabilistic classifiers (Section~\ref{sec:setup}) and derive a one-bit estimator of True
    Calibration Error. We prove its consistency and provide an upper bound of its convergence rate (Section~\ref{sec:onebit_calib}).
    \item Using randomized thresholds, we derive unbiased one-bit estimators
    of strictly proper scoring rules. We show that restricted access induces
    an additive variance penalty, which we characterize for the
    Brier score and Log-loss (Section~\ref{sec:onebit_psr}).

    \item We extend the framework to post-hoc recalibration by constructing a
    one-bit objective whose expectation matches that of standard
    Temperature Scaling, allowing the temperature to be fitted
    without observing the original probabilities (Section~\ref{sec:onebit_temperature}).

    \item We realize one-bit threshold queries for LLMs using
    \texttt{logit\_bias} interventions and assess the resulting calibration,
    proper-score, and recalibration procedures on controlled synthetic data
    and LLM benchmarks (Section~\ref{sec:experiments}).
\end{itemize}

\section{Background and Related Work}
\label{sec:background}
We focus on the standard binary classification setting. Let $(X,Y)$ be data drawn from an underlying distribution $\mathcal{D}$ with binary labels $Y\in\{0,1\}$, and let $S=f(X)\in(0,1)$ denote the predicted score of the positive class. In this section, we review three standard operations---measuring calibration error, evaluating proper scoring rules, and post-hoc recalibration---under the conventional assumption that $S$ is directly observed. These full-access procedures serve as baselines for the one-bit estimators developed in subsequent sections.

\subsection{Calibration Error}
\label{sec:bg_calibration}
A binary classifier is said to be calibrated if its predictions match the oracle conditional frequencies: $c_f(S)=S$ almost surely, where $c_f(s):=\mathbb{E}[Y\mid S=s]$. Several notions and measures of calibration have been proposed~\citep{vaicenavicius2019evaluating, Widmannetal2019, popordanoska2024proper, gruber2025optimizing}; in this work we focus  on the $L_1$ \emph{True Calibration Error} (TCE),
\begin{equation*}
    \mathrm{TCE}(f) := \mathbb{E}\big[|c_f(S)-S|\big].
\end{equation*}
A common finite-sample estimator of TCE is the \emph{Expected Calibration Error} (ECE) \citep{naeini2015obtaining,guo_calib}. Given a partition of $[0,1)$ denoted by  $B_m=[t_{m-1},t_m)$  for $m=1,\ldots,M$, it is defined as:
\begin{equation}
\label{eq:bg_ece}
    \widehat{\mathrm{ECE}}_M := \sum_{m=1}^M \left| \frac{1}{N} \sum_{i=1}^N (Y_i-S_i)\mathbf{1}\{S_i\in B_m\} \right|.
\end{equation}
Fixed-resolution ECE can suffer from discretization bias and intra-bin
cancellation~\citep{nixon2019measuring,kumar2019verified,
gupta2020distribution,roelofs2022mitigating}. Consistent estimation can instead
be obtained by letting the number of ECE bins grow with sample size
\citep{futami2024information}, or through kernel-based estimators
\citep{zhang2020mix,popordanoska2022consistent}. Our analysis is closest in
spirit to \citet{futami2024information}, who characterize the tradeoff between
bin resolution and estimation error under full observability; we study how this
tradeoff changes under our \emph{one-bit threshold access}.

\subsection{Proper Scoring Rules}
\label{sec:bg_psr}
Calibration alone does not characterize predictive quality: predicting the marginal event frequency can be perfectly calibrated while remaining uninformative. Strictly proper scoring rules instead evaluate the full probabilistic forecast~\citep{gneiting2007strictly}. Their expected risk
$R_\phi(f):=\mathbb{E}[\phi(S,Y)]$ admits the decomposition
\citep{Sanders1963OnSP,brocker2009reliability,kull2015novel}:
\begin{equation}
\label{eq:psr_decomposition}
    R_\phi(f)
    =
    \mathbb{E}\left[d_\phi(c_f(S),S)\right]
    +
    \mathbb{E}\left[H_\phi(c_f(S))\right]
\end{equation}
where the second term corresponds to refinement and the first to \emph{proper calibration
error}, respectively
\citep{popordanoska2024proper,gruber2025optimizing,berta2026variational}.
Since $L_1$ TCE is not a proper calibration error, the two criteria capture
distinct aspects of calibration quality.

We consider binary strictly proper scoring rules admitting an integral representation over elementary losses~\citep{schervish1989general,buja2005loss}:
\begin{equation}
\label{eq:psr_integral}
    \phi(s,y) = \int_0^1 w(t)L_t(s,y)\,\mathrm{d}t,
\end{equation}
for a non-negative weight function $w$ that characterizes the scoring rule and for
\begin{equation*}
    L_t(s,y) := y(1-t)\mathbf{1}\{s<t\} + (1-y)t\mathbf{1}\{s\geq t\}.
\end{equation*} Crucially, evaluating $L_t(S,Y)$ requires only knowing whether $S\geq t$. We exploit this representation in Section~\ref{sec:onebit_psr} to construct unbiased proper-score estimators from one-bit queries.

\subsection{Recalibration and Temperature Scaling}
\label{sec:bg_ts}

Post-hoc recalibration transforms model scores using a held-out calibration set; existing approaches include both parametric and non-parametric methods~\citep{platt1999probabilistic,zadrozny2001calibrated,naeini2015obtaining,guo_calib,berta26a}; see \citet{silvafilho2023survey,calarena2026} for broader reviews.

We focus on \emph{Temperature Scaling} (TS). In binary classification, TS applies the strictly increasing map
\begin{equation*}
    q_\alpha(s):=
    \sigma\!\big(\alpha\,\sigma^{-1}(s)\big),
    \qquad \alpha>0
\end{equation*}
where $\sigma$ denotes the sigmoid function. $\alpha$ is fitted under full score access by minimizing empirical Log-loss:
\begin{equation*}
    \widehat{\alpha} \in \arg\min_{\alpha\in\mathcal{A}} \frac{1}{N} \sum_{i=1}^N \ell_\alpha(S_i,Y_i).
\end{equation*}
Because $q_\alpha$ is bijective and monotone, TS preserves refinement~\citep{berta2025rethinking}; hence, Equation~\ref{eq:psr_decomposition} shows that a population-level Log-loss improvement comes exclusively from reduced proper calibration error. Section~\ref{sec:onebit_temperature} asks whether the same objective can be optimized without observing the scores directly.

\subsection{Restricted Black-Box Access}
\label{sec:bg_blackbox}

A distinct literature studies how much information can be recovered from models that expose only restricted outputs. \citet{tramer2016stealing} show that repeated queries to prediction APIs can be used to reconstruct or approximate black-box models, even when the API reveals only the predicted class rather than confidence scores. Similarly, \citet{sanyal2022towards} train a surrogate classifier using only the target model's predicted class for each queried input. More closely related to our setting, \citet{carlini2024stealing} show that repeated interventions through the \texttt{logit\_bias} parameter of language-model APIs can be used to recover otherwise hidden logits by binary search.

Our setting differs in both objective and query budget. Rather than using repeated queries to recover a model or reconstruct its continuous predictions, we ask whether probabilistic evaluation and recalibration can be performed directly from threshold comparisons. We restrict the evaluator to at most one threshold query per labeled example, excluding methods that require repeated queries to the API with the same input.

\section{Problem Setup and the One-Bit Access Model}
\label{sec:setup}

We formalize the \emph{one-bit threshold access} setting introduced above. The predicted score $S=f(X)\in(0,1)$ is hidden, but for any chosen threshold $t\in(0,1)$ the evaluator can observe: $Q_t(X) := \mathbf{1}\{S\geq t\}$.
For each labeled sample $(X_i,Y_i)$, the evaluator may choose a possibly randomized threshold $T_i$ and make at most one query, observing: $(Q_{T_i}(X_i), T_i, Y_i)$.
The boundary values $Q_0(X)=1$ and $Q_1(X)=0$ are known deterministically and require no model interaction.

The one-query restriction serves two purposes. First, it keeps the query budget
comparable to a standard evaluation pass, which typically uses one model
prediction per sample. Second, it rules out reconstructing individual
probabilities through repeated queries.

We study calibration-error estimation, proper-scoring-rule evaluation, and
post-hoc recalibration under this one-bit access model.

\paragraph{Realizing Threshold Queries in Language Models.}
We next describe a concrete realization of this access model when an LLM is
used as a zero-shot binary classifier. Given an input $X$, we formulate the task
as a binary question or statement whose answer is either true or false. The prompt instructs the model to return one of two
designated single-token labels, $v_+$ and $v_-$ (e.g., \textit{A}/\textit{B}),
corresponding to the two semantic answers. We assume that generation is
effectively restricted to these two labels, so that greedy decoding returns
the label with the larger model score.

Let $z_+(X)$ and $z_-(X)$ denote the corresponding \emph{logits}. Under this binary-answer assumption, the hidden probability of the positive answer is
\begin{equation*}
    S(X) := \frac{e^{z_+(X)}}{e^{z_+(X)}+e^{z_-(X)}} = \sigma\!\big(z_+(X)-z_-(X)\big).
\end{equation*}
Therefore, asking whether $S(X)$ exceeds a threshold $t\in(0,1)$ is equivalent to asking whether the logit difference exceeds the corresponding log-odds threshold:
\begin{equation*}
    S(X)\geq t \quad\Longleftrightarrow\quad z_+(X)-z_-(X)\geq \sigma^{-1}(t).
\end{equation*}

The \texttt{logit\_bias} API parameter~\citep{openai_logit_bias} allows the user to add a chosen constant to the logit of a specified output token before decoding. We exploit this to turn the desired probability threshold into the model's binary decision. Specifically, adding $b_t :=~-\sigma^{-1}(t)$
to the positive-token logit gives
\begin{align*}
    z_+(X)+b_t \geq z_-(X)
    &\Longleftrightarrow z_+(X)-z_-(X)\geq \sigma^{-1}(t) \\
    &\Longleftrightarrow S(X)\geq t.
\end{align*}
Hence, under the binary-answer and greedy-decoding assumptions, the model outputs $v_+$ exactly when $S(X)\geq t$. A single biased API call therefore reveals the threshold bit $Q_t(X) = \mathbf{1}\{S(X)\geq t\}$
without exposing either the logits or the continuous probability itself.

This provides a concrete LLM instantiation of one-bit threshold access. Practical details needed to enforce the binary-answer setup are given in Appendix~\ref{app:llm_details}. The statistical methods developed below depend only on access to $Q_t(X)$ and therefore apply beyond this particular realization.

\section{Theoretical Results}
\subsection{The One-Bit Calibration Estimator}
\label{sec:onebit_calib}
Under full probability access, each bin term in
Equation~\ref{eq:bg_ece} requires both knowing whether $S_i\in B_m$ and
evaluating the residual $Y_i-S_i$. Neither is directly available under
one-bit access.

We first replace the hidden score within bin $B_m=[t_{m-1},t_m)$ by its
midpoint $a_m:=(t_{m-1}+t_m)/2$. We then use 
\begin{equation*}
    \mathbf{1}\{S_i\in B_m\}
=
Q_{t_{m-1}}(X_i)-Q_{t_m}(X_i)
\end{equation*}
so that the corresponding bin contribution is approximated by a difference
between a lower-boundary term and an upper-boundary term.

Crucially, we do \emph{not} evaluate these two terms on the same observations.
We partition the data into disjoint subsets
$\mathcal{I}_0,\ldots,\mathcal{I}_{M-1}$ of the same cardinality. Samples in $\mathcal{I}_m$ are assigned
to the single threshold $t_{m}$, so no observation is queried at more than
one threshold. For bin $B_m$, the lower-boundary term is estimated using
samples from $\mathcal{I}_{m-1}$, while the upper-boundary term is estimated using
the different subset $\mathcal{I}_{m}$.

This gives
\begin{align}
\label{eq:onebit_bin_gap}
    \widehat{\Delta}_m^{\rm 1bit}
    ={}&
    \frac{1}{n_{m-1}}
    \sum_{i\in\mathcal{I}_{m-1}}
    (Y_i-a_m)Q_{t_{m-1}}(X_i)
    \nonumber\\
    &-
    \frac{1}{n_m}
    \sum_{i\in\mathcal{I}_m}
    (Y_i-a_m)Q_{t_m}(X_i)
\end{align}
with the endpoint terms handled using $Q_0=1$ and $Q_1=0$.
Summing their absolute values gives our one-bit estimator of TCE:
\begin{equation}
\label{eq:onebit_ce}
    \widehat{\mathrm{CE}}^{\mathrm{1bit}}_M
    :=
    \sum_{m=1}^M
    \left|\widehat{\Delta}_m^{\rm 1bit}\right|.
\end{equation}
\textbf{Consistency and Convergence Rate.}
We next show that the one-bit estimator consistently estimates TCE and give an upper bound on its convergence rate.

\begin{tcolorbox}[colback=gray!10, colframe=gray!50, arc=2pt, boxrule=1pt,
left=4pt, right=4pt, top=4pt, bottom=4pt]
\begin{theorem}[Consistency of One-Bit Evaluation]
\label{thm:onebit_consistency}
Let $(X_i,Y_i)_{i=1}^N \overset{\mathrm{i.i.d.}}{\sim}\mathcal{D}$ and let
$\widehat{\mathrm{CE}}^{\mathrm{1bit}}_M$ be the estimator defined in
Equation~\eqref{eq:onebit_ce}. Assume that $N\geq 2M$ and that the calibration
function $c_f$ is $L$-Lipschitz. Then there exists a constant $C_L<\infty$
such that: 
\begin{equation*}
    \mathbb{E}\left[
        \left(
            \widehat{\mathrm{CE}}^{\mathrm{1bit}}_M
            -
            \mathrm{TCE}
        \right)^2
    \right]
    \leq
    C_L\left(
        \frac{M^3}{N}
        +
        \frac{1}{M^2}
    \right),
\end{equation*}
where the expectation is over the sampled dataset.

Consequently, choosing $M\propto N^\alpha$ for any $0<\alpha<1/3$
guarantees consistency:
\begin{equation*}
    \lim_{N\to\infty}
    \mathbb{E}\left[
        \left(
            \widehat{\mathrm{CE}}^{\mathrm{1bit}}_M
            -
            \mathrm{TCE}
        \right)^2
    \right]
    =0.
\end{equation*}
\end{theorem}
\end{tcolorbox}
The proof is given in Appendix~\ref{app:proofs_calibration}. It separates the
error introduced by replacing hidden scores with bin midpoints from the
finite-sample error incurred when estimating adjacent threshold terms on
different subsets.

Balancing the two terms in our bound, $M^3/N$ and $M^{-2}$, gives
$M\asymp N^{1/5}$ and an MSE upper bound of
$\mathcal{O}(N^{-2/5})$. Under full probability access,
$\widehat{\mathrm{ECE}}_M$ admits the bound
$\mathcal{O}(M/N+M^{-2})$, which is balanced at
$M\asymp N^{1/3}$ and yields
$\mathcal{O}(N^{-2/3})$~\citep{futami2024information}. Thus the finite-sample term increases from
$M/N$ to $M^3/N$, leading to the slower rate
$\mathcal{O}(N^{-2/5})$. 

\textbf{Remark.} Although this is only an upper bound, both terms
have a natural explanation. Each threshold is estimated from only about
$N/M$ samples, giving noise of order $\sqrt{M/N}$ per bin. Since the
estimator sums absolute bin errors, this noise sums across $M$ bins,
suggesting squared error of order $M^3/N$. The $M^{-2}$ term comes from
replacing scores within a bin by its midpoint. This suggests that the bound
could be tight for our estimator and this will be empirically studied in Section~\ref{sec:exp_calibration}. Whether a different one-bit
estimator can achieve a faster rate remains an open question.
\subsection{Proper Scoring Rules from One-Bit Access}
\label{sec:onebit_psr}

Unlike the one-bit TCE estimator, proper-score risk admits an unbiased estimator
under one-bit access. As reviewed in Equation~\ref{eq:psr_integral}, a binary proper
scoring rule can be represented as an integral of elementary losses. Since
evaluating $L_t(S,Y)$ only requires knowing whether $S\geq t$, a single
threshold query is sufficient to evaluate the elementary loss at a sampled
threshold.

\textbf{The One-Bit Proper Score Estimator.}
We draw a threshold $T$ from a density
$g$ on $(0,1)$ (using the uniform $g(t)=1$ as our default) and observe $B := Q_T(X)=\mathbf{1}\{S\geq T\}.$
The corresponding elementary loss is then directly observable:
$L_T(S,Y)= Y(1-T)(1-B)+(1-Y)TB.$
We then define
\begin{equation*}
\widehat{\phi}_{\mathrm{1bit}}(B,Y,T)
    :=
    \frac{w(T)}{g(T)}L_T(S,Y),
\end{equation*}
where $L_T(S,Y)$ is evaluated from the observed bit $B$ using the expression
above.

For any fixed score-label pair $(S,Y)$, averaging over the random threshold
recovers the original proper score:
\begin{align*}
    \mathbb{E}_{T\sim g}
    \left[
        \widehat{\phi}_{\mathrm{1bit}}
        \mid S,Y
    \right]
    &=
    \int_0^1
        \frac{w(t)}{g(t)}
        L_t(S,Y)
        g(t)\,\mathrm{d}t \\
    &=
    \int_0^1
        w(t)L_t(S,Y)\,\mathrm{d}t =
    \phi(S,Y).
\end{align*}
Thus, one randomized threshold query gives an unbiased estimate of the proper
score. To estimate the population risk, we draw
independent thresholds $T_i\sim g$, and observe
$B_i=Q_{T_i}(X_i)$. We define
\begin{equation*}
\widehat{R}_{\phi,\mathrm{1bit}}
    :=
    \frac{1}{N}
    \sum_{i=1}^N
    \widehat{\phi}_{\mathrm{1bit}}(B_i,Y_i,T_i).
\end{equation*}
Averaging over both the sampled dataset and the random thresholds gives
\begin{equation*}
    \mathbb{E}_{\substack{
        (X_i,Y_i)_{i=1}^N\sim\mathcal{D}^N\\
        T_{1:N}\sim g^N
    }}
    \left[
        \widehat{R}_{\phi,\mathrm{1bit}}
    \right]
    =
    \mathbb{E}[\phi(S,Y)] = R_\phi(f).
\end{equation*}
Hence, the population proper-score risk can be estimated without bias using
one threshold query per example.

\textbf{Variance of the One-Bit Estimator.}
Because both the one-bit estimator and the standard probability-access
estimator are unbiased, we compare their variances. Define the latter as
$\widehat{R}_{\phi,\mathrm{full}}
    :=
    \frac{1}{N}\sum_{i=1}^N \phi(S_i,Y_i).$

\begin{tcolorbox}[colback=gray!10, colframe=gray!50, arc=2pt, boxrule=1pt,
left=4pt, right=4pt, top=4pt, bottom=4pt]
\begin{theorem}[Variance Cost]
\label{thm:psr_variance}
For $T\sim g$, define
$V_{\phi,g}(s,y)
    :=
    \operatorname{Var}_{T\sim g}
    \left(
        \widehat{\phi}_{\mathrm{1bit}}
        \bigl(\mathbf{1}\{s\geq T\},y,T\bigr)
    \right)$
and let
$J_{\phi,g}
    :=
    \mathbb{E}\left[V_{\phi,g}(S,Y)\right]$, then
\begin{align*}
    \operatorname{Var}\left(
        \widehat{R}_{\phi,\mathrm{1bit}}
    \right)
    &=
    \operatorname{Var}\left(
        \widehat{R}_{\phi,\mathrm{full}}
    \right)
    +
    \frac{J_{\phi,g}}{N}.
\end{align*}
provided that $\operatorname{Var}(\phi(S,Y))$ and $J_{\phi,g}$ are finite. Here, the first variance is over both the sampled dataset and the random
thresholds, while the second is over the sampled dataset.
\end{theorem}
\end{tcolorbox}

\textit{Proof.} Applying the law of total variance to $Z_i=\widehat{\phi}_{\mathrm{1bit}}(B_i,Y_i,T_i)$ conditioned on $(S_i,Y_i)$ gives $\operatorname{Var}(Z_i) = \operatorname{Var}(\phi(S_i,Y_i)) + J_{\phi,g}$. The result follows by averaging $N$ independent observations.

Thus, one-bit access adds the variance term $J_{\phi,g}/N$ relative to
probability access. The magnitude of  $J_{\phi,g}$ depends on the scoring rule, the threshold
distribution $g$, and the distribution of predictions and labels. We
illustrate this for the Brier score and Log-loss under uniform threshold
sampling, $g(t)=1$:

\begin{itemize}
    \item \textbf{Brier Score ($w(t)=2$)}. For $y=1$,
    $V_{\mathrm{Brier},1}(p,1)
        =
        \frac{4}{3}(1-p)^3-(1-p)^4
        \leq \frac{1}{3}.$
    By symmetry, the same bound holds for all $p$ and $y$, so
    $J_{\mathrm{Brier},1}\leq 1/3$. Hence, the additional variance from
    one-bit access is uniformly bounded, regardless of the prediction
    distribution.

    \item \textbf{Log-Loss ($w(t)=1/[t(1-t)]$)}.
    For $y=1$,
    $V_{\mathrm{Log},1}(p,1)
        =
        \frac{1}{p}-1-\log^2 p$.
    This quantity diverges as $p\to 0$, corresponding to increasingly
    confident errors when $y=1$. Thus, $J_{\mathrm{Log},1}$ can be large when
    such errors occur frequently. In contrast, if the predictor is perfectly
    calibrated, then $J_{\mathrm{Log},1}\leq 1$. We study this variance
    inflation empirically in Section~\ref{sec:experiments}.
\end{itemize}
\vspace{-0.25cm} Full proofs can be found in Appendix~\ref{app:psr_proofs_detailed}.

\subsection{Temperature Scaling from One-Bit Access}
\label{sec:onebit_temperature}

The previous section shows how to estimate a fixed proper scoring rule from a
single threshold query per sample. A natural approach for Temperature Scaling (TS) is to
apply the same construction to the Log-loss of the recalibrated prediction
$q_\alpha(S)$. For Log-loss, $w(u)=1/[u(1-u)]$, so
Equation~\ref{eq:psr_integral} gives
\begin{equation}
\label{eq:ts_logloss_integral}
    \ell_\alpha(S,Y)
    =
    \int_0^1
    \frac{1}{u(1-u)}
    L_u(q_\alpha(S),Y)\,\mathrm{d}u.
\end{equation}

Our queries threshold the original score $S$, rather than the transformed
score $q_\alpha(S)$. Using the change of variables $u=q_\alpha(t)$,
Equation~\ref{eq:ts_logloss_integral} becomes
\begin{equation}
\label{eq:ts_logloss_threshold}
    \ell_\alpha(S,Y)
    =
    \int_0^1
    \frac{\alpha}{t(1-t)}
    L_{q_\alpha(t)}(q_\alpha(S),Y)\,\mathrm{d}t.
\end{equation}
Since $q_\alpha$ is strictly increasing,
$\mathbf{1}\{q_\alpha(S)\geq q_\alpha(t)\}
=
\mathbf{1}\{S\geq t\}$.
A threshold query on $S$ therefore provides exactly the information needed to
evaluate the elementary loss in Equation~\ref{eq:ts_logloss_threshold}.

Drawing $T$ from a density $g$ on $(0,1)$ and observing
$B=\mathbf{1}\{S\geq T\}$, the corresponding elementary loss is directly
observable:
\begin{equation*}
    L_{q_\alpha(T)}(q_\alpha(S),Y)
    =
    Y(1-q_\alpha(T))(1-B)
    +
    (1-Y)q_\alpha(T)B.
\end{equation*}
We therefore define
\begin{equation}
\label{eq:ts_onebit_direct}
    \widehat{\ell}^{\rm 1bit}_{\alpha}(B,Y,T)
    :=
    \frac{\alpha}
         {g(T)T(1-T)}
    L_{q_\alpha(T)}(q_\alpha(S),Y).
\end{equation}
For a fixed $\alpha$, this estimator is unbiased for
$\ell_\alpha(S,Y)$. 

However, using it directly as an empirical objective creates
a finite-sample problem which we describe here. For any fixed $(B,Y,T)$,
$\widehat{\ell}^{\rm 1bit}_{\alpha}(B,Y,T)\to 0$ as $\alpha\to 0$ because of
the factor $\alpha$. Hence, for any finite dataset, the empirical risk
$\widehat{R}_{\rm 1bit}(\alpha)
:=N^{-1}\sum_{i=1}^N
\widehat{\ell}^{\rm 1bit}_{\alpha}(B_i,Y_i,T_i)$
also satisfies $\widehat{R}_{\rm 1bit}(\alpha)\to 0$ as $\alpha\to 0$.

The population TS risk behaves differently:
$R(\alpha)\to\log 2$, since $q_\alpha(S)\to 1/2$.
Thus, although the estimator is unbiased for every $\alpha>0$, the
finite-sample objective has the wrong limit near $\alpha=0$ and may select
values of $\alpha$ that are too small. This pathological finite-sample
behavior is discussed in more detail in Appendix~\ref{app:ts_theory}.

\textbf{A Centered One-Bit Objective.}
We avoid this problem by expressing the Log-loss relative to the fixed
prediction $1/2$. Since $q_\alpha(1/2)=1/2$ and
$\ell_\alpha(1/2,Y)=\log 2$ for every $\alpha>0$,
Equation~\ref{eq:ts_logloss_threshold} can be written as\par
{\small
\setlength{\abovedisplayskip}{0pt}
\setlength{\abovedisplayshortskip}{0pt}
\begin{align*}
    \ell_\alpha(&S,Y)-\log 2\\
    &=
    \int_0^1
    \frac{\alpha}{t(1-t)}
    \Big[
        L_{q_\alpha(t)}(q_\alpha(S),Y)
        -
        L_{q_\alpha(t)}(1/2,Y)
    \Big]
    \,\mathrm{d}t.
\end{align*}
}Drawing $T\sim g$ and observing
$B=\mathbf{1}\{S\geq T\}$, the integrand can be evaluated from the same
one-bit query. Simplifying the integrand gives the centered estimator:\par
{\footnotesize
\setlength{\abovedisplayskip}{0pt}
\setlength{\abovedisplayshortskip}{0pt}
\begin{equation}
\label{eq:ts_onebit_centered}
    \widehat{\ell}^{\rm 1bit,cent}_{\alpha}(B,Y,T)
    :=
    \log 2
    +
    \frac{\alpha\big(q_\alpha(T)-Y\big)}
         {g(T)T(1-T)}
    \left(
        B-\mathbf{1}\{T\leq 1/2\}
    \right)
\end{equation}%
}Once $(T,B,Y)$ has been observed, this quantity can be evaluated for any
$\alpha>0$ without additional model queries. We therefore define the centered
empirical objective
$\widehat{R}^{\rm cent}_{\rm 1bit}(\alpha)
:=N^{-1}\sum_{i=1}^N
\widehat{\ell}^{\rm 1bit,cent}_{\alpha}(B_i,Y_i,T_i)$
and fit $\alpha$ by minimizing it on the calibration dataset.

\begin{tcolorbox}[colback=gray!10, colframe=gray!50, arc=2pt, boxrule=1pt, left=4pt, right=4pt, top=4pt, bottom=4pt]
\begin{proposition}[Properties of the Centered TS Objective]
\label{prop:centered_ts}
Let
$R(\alpha):=\mathbb{E}[\ell_\alpha(S,Y)]$
denote the population Log-loss at inverse temperature $\alpha$. Then:
\begin{enumerate}
    \item \textbf{Unbiasedness:} for every $\alpha>0$,
    $\mathbb{E}_{T\sim g}\!\left[
    \widehat{\ell}^{\rm 1bit,cent}_{\alpha}(B,Y,T)
    \mid S,Y\right]
    =\ell_\alpha(S,Y)$.

    \item \textbf{Correct behavior at zero:} for any finite sample,
    $\lim_{\alpha\to 0}
    \widehat{R}^{\rm cent}_{\rm 1bit}(\alpha)
    =\log 2$.

    \item \textbf{Same expected objective:}
    $\mathbb{E}\!\left[
    \widehat{R}^{\rm cent}_{\rm 1bit}(\alpha)
    \right]
    =R(\alpha)$.
\end{enumerate}
Therefore, the expected centered one-bit objective and the standard
probability-access objective have the same population minimizers.
\end{proposition}
\end{tcolorbox}
Proposition~\ref{prop:centered_ts} (proofs detailed in Appendix~\ref{app:ts_theory}) should be viewed as an unbiased-objective
construction rather than a characterization of the empirical minimizer.
We do not establish consistency for
$\widehat{\alpha}$, and instead evaluate its practical behavior empirically in
Section~\ref{sec:exp_temperature_scaling}.
This construction extends beyond temperature scaling to any differentiable monotone recalibration
maps; see Appendix~\ref{app:general_recalibration} for further details.
\section{Experiments}
\label{sec:experiments}
\begin{figure*}[t]
    \centering
    \includegraphics[width=0.98\textwidth]{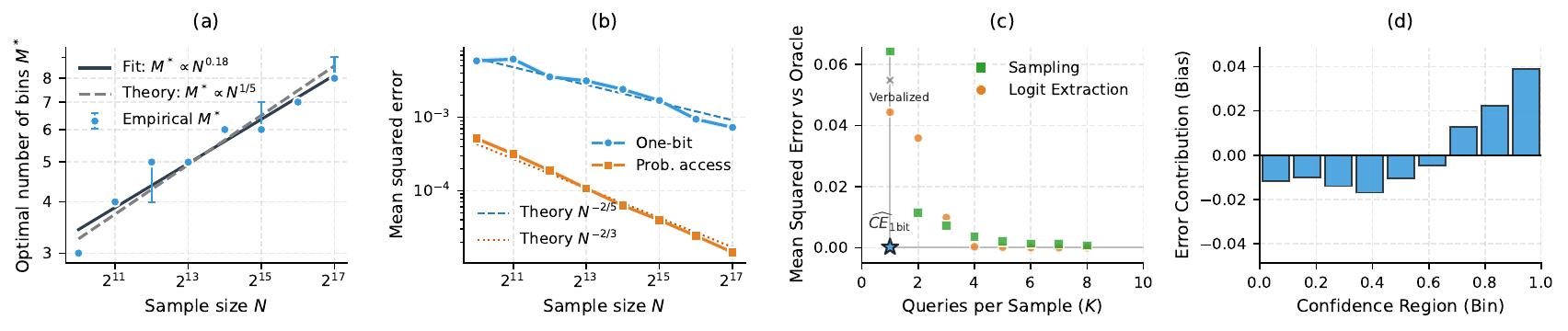}
    \caption{\textbf{One-bit calibration auditing.}
    \textbf{(a)} Empirical MSE-optimal resolution $M^\star$ versus $N$ for $A=4$, compared with the predicted $N^{1/5}$ scaling.
    \textbf{(b)} MSE convergence for $A=2$ under one-bit and probability access, with reference rates $N^{-2/5}$ and $N^{-2/3}$.
    \textbf{(c)} Query cost versus calibration-error estimation accuracy across four LLMs.
    \textbf{(d)} Signed one-bit calibration contributions across score intervals for Mistral-7B on MMLU, compared with the white-box reference.}
    \label{fig:calibration_main}
\end{figure*}

To test our theoretical results, we evaluate the one-bit framework on both
controlled synthetic data and LLM experiments emulating restricted black-box access.

\textbf{Controlled Synthetic Setup.}
Throughout the synthetic experiments, we use
\begin{equation*}
\begin{aligned}
    U &\sim \mathrm{Uniform}[-1,1], \quad S = \sigma(AU), \\
    Y \mid U &\sim \mathrm{Bernoulli}(\sigma(\kappa U))
\end{aligned}
\end{equation*}
with $A,\kappa>0$. The calibration function is
$c_f(s)=\sigma\!\left(\frac{\kappa}{A}\sigma^{-1}(s)\right)$.
Thus, $A=\kappa$ gives a perfectly calibrated predictor, while $A>\kappa$
produces increasingly overconfident scores without changing the decision
boundary. This family provides a closed-form TCE, controlled confident-error
severity, and population-optimal inverse temperature
$\alpha^\star=\kappa/A$. Unless stated otherwise, we fix $\kappa=2$ and vary
$A$.

\subsection{Calibration Evaluation}
\label{sec:exp_calibration}

\textbf{Synthetic Scaling and Convergence.}
We empirically study two predictions of our theoretical analysis. First, for
$A=4$, we sweep $M\in\{2,\ldots,40\}$ and measure how the empirical MSE-minimizing
resolution $M^\star$ scales with $N$. Second, in the calibrated setting (where $\mathrm{TCE}=0$), we study MSE convergence using the theoretically
prescribed resolution scaling. For each $N$, all reported MSEs are averaged
over $500$ trials. 

Figure~\ref{fig:calibration_main}(a) shows that the empirical optimum follows
$M^\star\propto N^{\hat\beta}$ with $\hat\beta\simeq0.18$, close to the
theoretical exponent $1/5=0.20$. Figure~\ref{fig:calibration_main}(b) shows
that, in the calibrated setting, the one-bit MSE closely follows the
$N^{-2/5}$ reference slope from our upper bound, while probability access
follows the faster $N^{-2/3}$ scaling of $\widehat{\rm ECE}_M$
\citep{futami2024information}. These results support the scaling predicted by
Theorem~\ref{thm:onebit_consistency}. 
Full synthetic experiments are provided in Appendix~\ref{app:synthetic_curves}

\textbf{LLM Calibration Evaluation.}
We evaluate the one-bit estimator $\widehat{\mathrm{CE}}^{\mathrm{1bit}}_M$
on BoolQ~\citep{clark2019boolq} and MMLU~\citep{mmlu}. BoolQ is naturally
binary, while each four-choice MMLU question is converted into four binary
candidate questions. We emulate a restricted black-box API using four
open-weight instruction-tuned models:
Qwen-2.5-7B-Instruct~\citep{qwen2025qwen25technicalreport},
Llama-3.1-8B-Instruct~\citep{grattafiori2024llama3herdmodels},
Mistral-7B-Instruct-v0.3~\citep{jiang2023mistral7b}, and
Gemma-2-9B-IT~\citep{gemmateam2024gemma2improvingopen}.
Implementation details and validation of the \texttt{logit\_bias} mechanism are
given in Appendix~\ref{app:llm_details}. The hidden continuous scores are used
only for the probability-access reference.

\emph{Baselines.}
We compare against three restricted-access alternatives: Verbalized
Confidence, which asks the model to report a numerical confidence in one call;
Sampling Frequency, which estimates probabilities from $K$ sampled outputs;
and Logit Extraction~\citep{carlini2024stealing}, which uses
$K$ repeated \texttt{logit\_bias} queries to estimate the hidden logit
difference. We use $K\in\{1,\ldots,8\}$ for the repeated-query baselines and
\texttt{temperature}=1.0 for sampling. Full implementation details of baselines are given
in Appendix~\ref{app:llm_details}. We report the Mean Squared Error (MSE)
between each estimate and the probability-access $\widehat{\mathrm{ECE}}_M$,
with $M=N^{1/3}$.

\emph{Cost--Error Trade-off.}
Figure~\ref{fig:calibration_main}(c) shows that, at one query per example, the
one-bit estimator has substantially lower MSE than Verbalized Confidence.
Sampling improves with additional queries but requires repeated model
calls to reach errors comparable to the single-query one-bit estimator. Logit extraction reaches comparable error only
after roughly $4$--$5$ queries per example. Thus,
$\widehat{\mathrm{CE}}^{\mathrm{1bit}}_M$ lies on the query--error Pareto
frontier among the evaluated methods.

\emph{Calibration Profile.}
The estimator also provides signed bin-level contributions rather than only an
aggregate calibration error. Figure~\ref{fig:calibration_main}(d) shows these
contributions for Mistral-7B on MMLU, with additional examples in
Appendix~\ref{app:calib_llm_xp}. For each interval, the plotted quantity
is $\widehat{\Delta}_m^{\mathrm{1bit}}$ from
Equation~\ref{eq:onebit_bin_gap}. Positive values indicate underconfidence and
negative values indicate overconfidence.
\subsection{Proper Scoring Rules}
\label{sec:exp_psr}

\begin{figure*}[t]
    \centering
    \includegraphics[width=\textwidth]{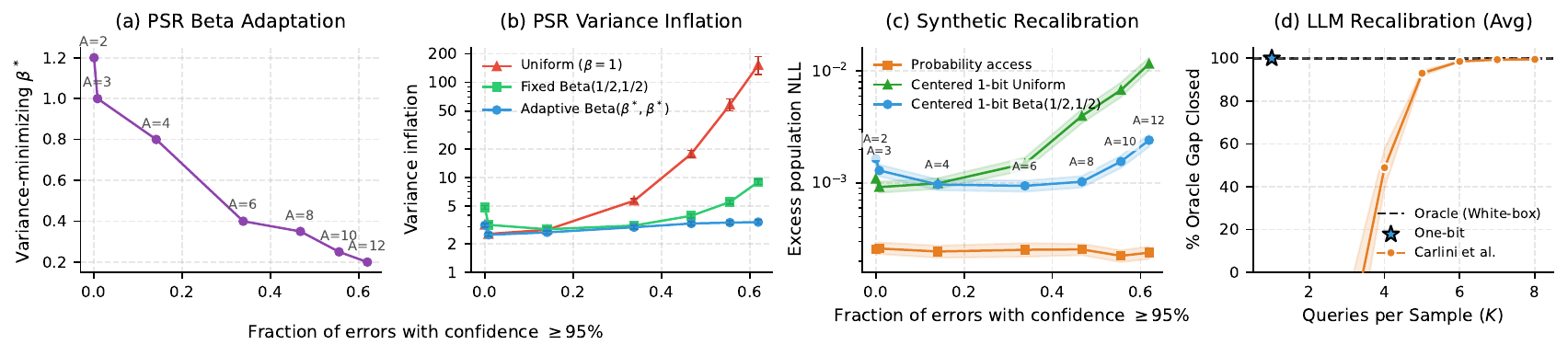}
    \caption{
    \textbf{One-bit proper-score estimation and Temperature Scaling.}
    \textbf{(a)} Oracle variance-minimizing $\beta^\star$ within the
    $\mathrm{Beta}(\beta,\beta)$ family as confident errors become more common.
    \textbf{(b)} Variance inflation of one-bit Log-loss relative to probability
    access; uniform sampling degrades quickly, while
    $\mathrm{Beta}(1/2,1/2)$ remains stable.
    \textbf{(c)} Excess population Log-loss after TS on the synthetic family;
    centered one-bit TS with $\mathrm{Beta}(1/2,1/2)$ closely tracks
    probability-access TS.
    \textbf{(d)} Average LLM recalibration efficiency versus API queries per
    calibration example. One-bit TS recovers over $98\%$ of the
    probability-access TS improvement with one query, while the Carlini et al.\
    baseline requires $K\geq5$ queries for comparable performance.
    }
    \label{fig:main_psr_ts}
\end{figure*}

Beyond calibration, proper scoring rules measure the overall quality of
predictions. Since our one-bit estimator is unbiased, its practical performance
depends on the added variance penalty $J_{\phi,g}$ from
Theorem~\ref{thm:psr_variance}.

\textbf{Threshold sampling under confident errors.}
For the Brier score, uniform threshold sampling has a bounded variance penalty
($J_{\mathrm{Brier},1}\leq 1/3$ see Section~\ref{sec:onebit_psr}). We therefore focus on Log-loss, whose
one-bit variance can become large when the model makes highly confident errors.
This is also the loss used for Temperature Scaling.

Using the common synthetic family, we fix $\kappa=2$ and vary
$A\in\{2,3,4,6,8,10,12\}$. Since $\widehat Y:=\mathbf{1}\{S\geq1/2\}=\mathbf{1}\{U\geq0\}$,
changing $A$ leaves the predicted labels unchanged while making the scores more
extreme. We measure this with
\begin{equation*}
    \rho_{.95}
    :=
    \mathbb{P}\!\left(
        \max(S,1-S)\geq0.95
        \mid \widehat Y\neq Y
    \right),
\end{equation*}
the fraction of classification errors made with at least $95\%$ confidence.
This quantity is used only as a synthetic diagnostic and is not assumed to be
available in the general case.

For Log-loss, Section~\ref{sec:onebit_psr} showed that Uniform threshold sampling can
have large variance in the presence of confident errors. Since such errors
occur when predicted scores lie near $0$ or $1$, placing more threshold mass near
the endpoints implies using more queries to the regions that contribute most to
this variance. We therefore consider symmetric proposals
$T\sim\mathrm{Beta}(\beta,\beta)$, where $\beta<1$ increasingly concentrates
thresholds near the endpoints.

Figure~\ref{fig:main_psr_ts}(a) shows that $\beta^\star$ decreases as confident
errors become more common, placing more mass near the endpoints.
Figure~\ref{fig:main_psr_ts}(b) shows that uniform sampling becomes much less
efficient in this regime, while the fixed $\mathrm{Beta}(1/2,1/2)$ proposal
remains more stable. The oracle $\beta^\star$ is mainly a synthetic diagnostic: it shows how the
best proposal changes with the error regime. In principle, $\beta$ could also
be tuned in real-world settings using additional data or prior information,
but this adds another estimation step.

We therefore use and recommend the fixed $\mathrm{Beta}(1/2,1/2)$ proposal as
a simple compromise. It puts more mass near the endpoints than uniform
sampling, requires no access to the hidden scores, and performs well across
the regimes we consider. We use this fixed proposal in our real-world
experiments below. Additional synthetic and LLM experiments are provided in Appendices~\ref{app:psr_experiments} and \ref{app:psr_llm}.

\subsection{Temperature Scaling Recalibration}
\label{sec:exp_temperature_scaling}

We test whether the centered one-bit objective can recover the performance of
standard Temperature Scaling (TS). We compare probability-access TS with
one-bit TS using Uniform and $\mathrm{Beta}(1/2,1/2)$ threshold proposals,
with one threshold query per calibration example.

\textbf{Synthetic recalibration.}
Using the synthetic family with $\kappa=2$,  $\alpha^\star=2/A$. For each $A$, we fit TS on
$N_{\mathrm{cal}}=2048$ examples over $500$ trials.
Figure~\ref{fig:main_psr_ts}(c) reports the excess population Log-loss,
$R(\widehat{\alpha})-R(\alpha^\star)$. When $A\leq6$, both one-bit
proposals closely follow probability-access TS. As confident errors become
more common, Uniform sampling degrades, while
$\mathrm{Beta}(1/2,1/2)$ remains much more stable. This matches the
proper-score results in Section~\ref{sec:exp_psr}.

\textbf{LLM recalibration.}
We next apply one-bit TS to the four LLMs on BoolQ and MMLU. Across $100$
random splits, we use $N_{\mathrm{cal}}=1024$ original questions for
calibration and evaluate on a disjoint held-out set.
Figure~\ref{fig:main_psr_ts}(d) reports the fraction of the
probability-access TS improvement recovered by each method. With one query
per calibration example, one-bit TS recovers over $98\%$ of the
probability-access improvement on average. In comparison, Carlini-style
logit extraction requires about $5-6$ queries per example to reach a similar
level. Thus, one-bit TS nearly matches standard TS without observing the
underlying probabilities.

Importantly, the fitted parameter can also be used at deployment time.
Temperature is commonly exposed as an API parameter, so after estimating the
inverse temperature $\widehat{\alpha}$ from one-bit queries, the black-box
model can be queried using temperature $1/\widehat{\alpha}$. This allows the
same model to be deployed with the learned recalibration without requiring
access to its logits or probabilities. Additional synthetic and LLM experiments are provides in Appendices~\ref{app:ts_experiments} and \ref{app:ts_llms}.

\section{Discussion and Conclusion}
\label{sec:conclusion}

This work shows that access to predicted probabilities is not necessary for
several standard probabilistic tasks. Under \emph{one-bit threshold access},
we can estimate calibration error consistently, estimate proper scores without
bias, and perform recalibration, all without reconstructing individual
probabilities.

The restriction nevertheless has a statistical cost. For calibration error,
our estimator converges more slowly;
for proper scores, it introduces additional variance, which can be
substantial for Log-loss in the presence of confident errors. The choice of
threshold distribution can mitigate this cost.

For LLMs used as binary classifiers, the required threshold comparisons can be
implemented through \texttt{logit\_bias}, and a temperature fitted from one-bit
observations can then be used at deployment. More broadly, these results suggest that probabilistic
evaluation remains possible even when model interfaces reveal substantially
less than the underlying probabilities.

\clearpage
\bibliography{bib}

\clearpage
\appendix
\thispagestyle{empty}

\onecolumn

\section{Limitations}
\label{app:limitations}

Our framework has several limitations. First, we focus on binary probabilistic
classification and our proper-scoring rule estimator construction relies on the integral
representation of binary proper scoring rules, which does not extend in the
same general form to multiclass proper losses~\citep{williamson2016composite}.
Multiclass one-bit evaluation may therefore require different constructions.

Second, the probability-access ECE used in our LLM experiments is a reference, not
the ground-truth TCE. As a finite-sample binned estimator, it is affected by
sampling error and discretization. Our consistency result instead concerns the
population TCE.

Third, our methods assume access to threshold comparisons. For LLMs, we realize
this through \texttt{logit\_bias} and a designated binary answer
format, which is not available in every black-box API.

Finally, our theoretical results do not establish optimality. The TCE result
assumes a Lipschitz calibration function, and the derived rates are upper
bounds for our estimator. With a larger fixed query budget, it remains open whether extra queries should be spread across examples or concentrated on selected examples at multiple thresholds. For proper-score estimation, the $\rm Beta$ threshold
proposal is heuristic, and no single proposal is optimal across all
prediction distributions.

\section{Proofs and Additional Theory}
\label{app:theory}

\subsection{One-Bit Calibration Error}
\label{app:proofs_calibration}

We prove Theorem~\ref{thm:onebit_consistency}. Throughout, let
$S=f(X)\in(0,1)$ denote the hidden score. We partition $[0,1]$ into
$M$ equal-width bins
\[
    B_m=[t_{m-1},t_m),
    \qquad
    t_m=\frac{m}{M},
    \qquad
    a_m=\frac{t_{m-1}+t_m}{2}.
\]
The data points are partitioned into disjoint subsets
$\mathcal I_0,\ldots,\mathcal I_{M-1}$ with sizes
$n_j:=|\mathcal I_j|
\in\{\lfloor N/M\rfloor,\lceil N/M\rceil\}$.
Samples in $\mathcal I_j$ are assigned to threshold $t_j$.

For convenience, define
\[
    W_{m,t}(X,Y)
    :=
    (Y-a_m)Q_t(X)
    =
    (Y-a_m)\mathbf 1\{S\geq t\}.
\]
For $m=1,\ldots,M-1$, the local one-bit contribution is
\begin{equation}
\label{eq:app_onebit_gap}
    \widehat{\Delta}_m^{\rm 1bit}
    =
    \frac{1}{n_{m-1}}
    \sum_{i\in\mathcal I_{m-1}}
    W_{m,t_{m-1}}(X_i,Y_i)
    -
    \frac{1}{n_m}
    \sum_{i\in\mathcal I_m}
    W_{m,t_m}(X_i,Y_i).
\end{equation}
For the last bin, the second term vanishes because $Q_1=0$:
\[
    \widehat{\Delta}_M^{\rm 1bit}
    =
    \frac{1}{n_{M-1}}
    \sum_{i\in\mathcal I_{M-1}}
    W_{M,t_{M-1}}(X_i,Y_i).
\]
The estimator is
\[
    \widehat{\mathrm{CE}}^{\mathrm{1bit}}_M
    =
    \sum_{m=1}^M
    \left|\widehat{\Delta}_m^{\rm 1bit}\right|.
\]

\subsubsection{Mean and Variance of the Bin Contributions}

Define the population midpoint contribution
\[
    \Delta_m^{\rm mid}
    :=
    \mathbb E\!\left[
        (Y-a_m)\mathbf 1\{S\in B_m\}
    \right].
\]

\begin{lemma}[Mean of the bin contribution]
\label{lem:onebit_gap_mean}
For every $m=1,\ldots,M$,
\[
    \mathbb E[
        \widehat{\Delta}_m^{\rm 1bit}
    ]
    =
    \Delta_m^{\rm mid}.
\]
\end{lemma}

\begin{proof}
For $m<M$, linearity of expectation gives
\begin{align*}
\mathbb E[
    \widehat{\Delta}_m^{\rm 1bit}
]
&=
\mathbb E[
    (Y-a_m)\mathbf 1\{S\geq t_{m-1}\}
]-
\mathbb E[
    (Y-a_m)\mathbf 1\{S\geq t_m\}
]
\\
&=
\mathbb E[
    (Y-a_m)\mathbf 1\{S\in B_m\}
]
\\
&=\Delta_m^{\rm mid}.
\end{align*}
For $m=M$, the same identity holds because $Q_1=0$.
\end{proof}

\begin{lemma}[Variance of the bin contribution]
\label{lem:onebit_gap_variance}
If $N\geq2M$, then
\[
    \operatorname{Var}
    \left(
        \widehat{\Delta}_m^{\rm 1bit}
    \right)
    \leq
    \frac{4M}{N}
\]
for every $m$.
\end{lemma}

\begin{proof}
For $m<M$, the two sample means in
Equation~\ref{eq:app_onebit_gap} are computed from disjoint subsets and are
therefore independent. Hence,
\begin{align*}
\operatorname{Var}
\left(
    \widehat{\Delta}_m^{\rm 1bit}
\right)
&=
\frac{
    \operatorname{Var}(W_{m,t_{m-1}}(X,Y))
}{n_{m-1}}
+
\frac{
    \operatorname{Var}(W_{m,t_m}(X,Y))
}{n_m}.
\end{align*}
Since $Y,a_m,Q_t(X)\in[0,1]$, we have
$|W_{m,t}(X,Y)|\leq1$ and therefore
$\operatorname{Var}(W_{m,t})\leq1$.
Moreover, $N\geq2M$ implies
$n_j\geq\lfloor N/M\rfloor\geq N/(2M)$.
Thus,
\[
    \operatorname{Var}
    \left(
        \widehat{\Delta}_m^{\rm 1bit}
    \right)
    \leq
    \frac{4M}{N}.
\]
The last bin contains only one random sample mean and satisfies the same bound.
\end{proof}

\subsubsection{Bias Decomposition}

Define the population contribution based on the true score,
\[
    \Delta_m^\star
    :=
    \mathbb E[
        (Y-S)\mathbf 1\{S\in B_m\}
    ]
    =
    \mathbb E[
        (c_f(S)-S)\mathbf 1\{S\in B_m\}
    ].
\]
The TCE can be written as
\[
    \mathrm{TCE}
    =
    \sum_{m=1}^M
    \mathbb E\!\left[
        |c_f(S)-S|
        \mathbf 1\{S\in B_m\}
    \right].
\]

Adding and subtracting
$\sum_m|\Delta_m^{\rm mid}|$ and
$\sum_m|\Delta_m^\star|$ gives
\begin{equation}
\label{eq:app_bias_decomp}
\mathbb E[
    \widehat{\mathrm{CE}}^{\mathrm{1bit}}_M
]
-
\mathrm{TCE}
=
\mathcal B_{\rm stat}
+
\mathcal B_{\rm mid}
-
\mathcal B_{\rm bin},
\end{equation}
where
\begin{align*}
\mathcal B_{\rm stat}
&:=
\sum_{m=1}^M
\left(
    \mathbb E[
        |\widehat{\Delta}_m^{\rm 1bit}|
    ]
    -
    |\Delta_m^{\rm mid}|
\right),\\
\mathcal B_{\rm mid}
&:=
\sum_{m=1}^M
\left(
    |\Delta_m^{\rm mid}|
    -
    |\Delta_m^\star|
\right),\\
\mathcal B_{\rm bin}
&:=
\sum_{m=1}^M
\left(
    \mathbb E[
        |c_f(S)-S|
        \mathbf 1\{S\in B_m\}
    ]
    -
    |\Delta_m^\star|
\right).
\end{align*}
These terms correspond respectively to finite-sample error, midpoint
approximation, and intra-bin cancellation.

\paragraph{Finite-sample term.}
By Lemma~\ref{lem:onebit_gap_mean},
\[
\mathcal B_{\rm stat}
=
\sum_{m=1}^M
\left(
    \mathbb E[
        |\widehat{\Delta}_m^{\rm 1bit}|
    ]
    -
    \left|
        \mathbb E[
            \widehat{\Delta}_m^{\rm 1bit}
        ]
    \right|
\right).
\]
For any random variable $Z$,
\[
    \mathbb E[|Z|]-|\mathbb E[Z]|
    \leq
    \sqrt{\operatorname{Var}(Z)}.
\]
Using Lemma~\ref{lem:onebit_gap_variance},
\begin{equation}
\label{eq:app_stat_bias}
    0
    \leq
    \mathcal B_{\rm stat}
    \leq
    2M\sqrt{\frac{M}{N}}
    =
    2\frac{M^{3/2}}{N^{1/2}}.
\end{equation}

\paragraph{Midpoint term.}
By the reverse triangle inequality,
\begin{align*}
\left|
    |\Delta_m^{\rm mid}|
    -
    |\Delta_m^\star|
\right|
&\leq
\left|
    \Delta_m^{\rm mid}
    -
    \Delta_m^\star
\right|\\
&\leq
\mathbb E\!\left[
    |S-a_m|
    \mathbf 1\{S\in B_m\}
\right].
\end{align*}
Since $a_m$ is the midpoint of a bin of width $1/M$,
$|S-a_m|\leq1/(2M)$ whenever $S\in B_m$. Therefore,
\begin{equation}
\label{eq:app_mid_bias}
    |\mathcal B_{\rm mid}|
    \leq
    \frac{1}{2M}.
\end{equation}

\paragraph{Binning term.}
The term $\mathcal B_{\rm bin}$ is the within-bin cancellation error also
present in probability-access ECE. Let
$h(s):=c_f(s)-s$. If $c_f$ is $L$-Lipschitz, then $h$ is
$(L+1)$-Lipschitz. Standard bounds for binned calibration estimators
therefore give~\citep{futami2024information}
\begin{equation}
\label{eq:app_bin_bias}
    0
    \leq
    \mathcal B_{\rm bin}
    \leq
    \frac{L+1}{M}.
\end{equation}

\subsubsection{Variance of the Full Estimator}

\begin{lemma}[Variance of the one-bit estimator]
\label{lem:onebit_ce_variance}
If $N\geq2M$, then
\[
    \operatorname{Var}
    \left(
        \widehat{\mathrm{CE}}^{\mathrm{1bit}}_M
    \right)
    \leq
    \frac{32M^2}{N}.
\]
\end{lemma}

\begin{proof}
Let
\[
    \Phi(\mathcal D_N)
    :=
    \widehat{\mathrm{CE}}^{\mathrm{1bit}}_M.
\]
Replace one observation $(X_k,Y_k)$ by an independent copy. Since every
observation belongs to exactly one subset $\mathcal I_j$, this replacement
affects at most two adjacent bin contributions.

For each affected bin, replacing one observation changes the corresponding
sample mean by at most
\[
    \frac{2}{n_j}
    \leq
    \frac{4M}{N},
\]
because $|W_{m,t}|\leq1$. Using
$\big||a|-|b|\big|\leq|a-b|$, the total change in $\Phi$ is therefore at most
\[
    \frac{8M}{N}.
\]
The Efron--Stein inequality then gives
\begin{align*}
\operatorname{Var}(\Phi)
&\leq
\frac12
\sum_{k=1}^N
\mathbb E[
    (\Phi(\mathcal D_N)-\Phi(\mathcal D_N^{(k)}))^2
]
\\
&\leq
\frac12 N
\left(
    \frac{8M}{N}
\right)^2
=
\frac{32M^2}{N}.
\end{align*}
\end{proof}

\subsubsection{Proof of Theorem~\ref{thm:onebit_consistency}}

From Equations~\ref{eq:app_bias_decomp}--\ref{eq:app_bin_bias},
\begin{align*}
\left|
    \mathbb E[
        \widehat{\mathrm{CE}}^{\mathrm{1bit}}_M
    ]
    -
    \mathrm{TCE}
\right|
&\leq
\mathcal B_{\rm stat}
+
|\mathcal B_{\rm mid}|
+
\mathcal B_{\rm bin}\\
&\leq
C_L
\left(
    \frac{M^{3/2}}{N^{1/2}}
    +
    \frac{1}{M}
\right)
\end{align*}
for a constant $C_L$ depending only on $L$. Hence,
\[
    \operatorname{Bias}^2
    \leq
    C_L
    \left(
        \frac{M^3}{N}
        +
        \frac{1}{M^2}
    \right).
\]

Using Lemma~\ref{lem:onebit_ce_variance} and the bias--variance decomposition,
\begin{align*}
\mathbb E\!\left[
    \left(
        \widehat{\mathrm{CE}}^{\mathrm{1bit}}_M
        -
        \mathrm{TCE}
    \right)^2
\right]
&=
\operatorname{Bias}^2
+
\operatorname{Var}
\left(
    \widehat{\mathrm{CE}}^{\mathrm{1bit}}_M
\right)\\
&\leq
C_L
\left(
    \frac{M^3}{N}
    +
    \frac{1}{M^2}
\right),
\end{align*}
where the $M^2/N$ variance term is absorbed by $M^3/N$ for $M\geq1$.

Finally, if $M\propto N^\alpha$, then
\[
    \frac{M^3}{N}
    +
    \frac{1}{M^2}
    =
    \mathcal O
    \left(
        N^{3\alpha-1}
        +
        N^{-2\alpha}
    \right).
\]
Both terms vanish whenever $0<\alpha<1/3$, proving consistency.

\paragraph{Lipschitz assumption.}
The smoothness assumption in Theorem~\ref{thm:onebit_consistency} concerns the
calibration function $c_f(s)=\mathbb E[Y\mid S=s]$, not the predictor $f$
itself. It controls how much the calibration residual
$c_f(s)-s$ can vary within a bin and is used to bound the binning error.
Related smoothness assumptions are standard in finite-sample analyses of
calibration-error estimators
\citep{popordanoska2022consistent,futami2024information}.
\subsubsection{Derivation of the Optimal Bin Scaling}
\label{app:optimal_bins}

From Theorem~\ref{thm:onebit_consistency}, the upper bound on the Mean Squared Error (MSE) of the one-bit estimator is characterized by:
\begin{equation}
    \operatorname{MSE} \le \mathcal{O}\left( \frac{M^3}{N} + \frac{1}{M^2} \right)
\end{equation}
To determine the optimal scaling of the number of bins $M$ with respect to the dataset size $N$ that minimizes this macroscopic error bound, we differentiate the bound with respect to $M$ and set the derivative to zero:
\begin{equation}
    \frac{\partial}{\partial M} \left( \frac{M^3}{N} + \frac{1}{M^2} \right) = \frac{3M^2}{N} - \frac{2}{M^3} = 0
\end{equation}
Solving for $M$, we obtain:
\begin{equation}
    \frac{3M^2}{N} = \frac{2}{M^3} \implies M^5 = \frac{2N}{3} \implies M \propto N^{1/5}
\end{equation}
Thus, to asymptotically minimize the theoretical estimation error while balancing the one-bit subset variance against the binning discretization bias, the optimal number of bins should scale as $\mathcal{O}(N^{1/5})$.

\paragraph{Interpretation of the rate.}
The proof also clarifies the origin of the two terms in the MSE bound.
Because the data are split across $M$ thresholds, each local contribution is
estimated from about $N/M$ observations and therefore has standard deviation
of order $\sqrt{M/N}$. The absolute value in
$\widehat{\mathrm{CE}}^{\mathrm{1bit}}_M$ then introduces a finite-sample
bias of the same order for each bin:
\[
    \mathbb{E}\!\left[
        |\widehat{\Delta}_m^{\rm 1bit}|
    \right]
    -
    |\Delta_m^{\rm mid}|
    =
    \mathcal{O}\!\left(\sqrt{\frac{M}{N}}\right).
\]
Summing this effect over $M$ bins gives a statistical bias of order
$M^{3/2}/\sqrt{N}$, whose square yields the $M^3/N$ term in the MSE bound.
Notice that this term is not the variance of the full estimator, which our
Efron--Stein argument bounds by the smaller order $M^2/N$.

The second term comes from approximating the continuous calibration error by
binwise quantities. Replacing the hidden score by the bin midpoint contributes
a bias of order $1/M$, while the possible cancellation of positive and
negative calibration errors within a bin is also of order $1/M$ under the
Lipschitz assumption on $c_f$. Together, these approximation errors lead to
the $M^{-2}$ term after squaring.

Balancing $M^3/N$ and $M^{-2}$ gives $M\asymp N^{1/5}$ and the resulting
MSE upper bound $\mathcal{O}(N^{-2/5})$. These arguments explain why this rate
arises for our estimator, but they do not establish a matching lower bound.
Empirically, Section~\ref{sec:exp_calibration} finds both an MSE-optimal
resolution close to $N^{1/5}$ and MSE scaling close to $N^{-2/5}$ in the
settings we study, suggesting that the bound captures the observed behavior
of the estimator. Whether it is tight more generally, or whether a different
one-bit estimator can achieve a faster rate, remains open.
\subsection{Proper Scoring Rules}
\label{app:psr_proofs_detailed}

This section gives additional details for the one-bit proper-score estimator.
We recall the integral representation used in the main text, prove
unbiasedness of the population-risk estimator, derive the variance
decomposition in Theorem~\ref{thm:psr_variance}, and compute the variance
penalty for the Brier score and Log-loss.

\subsubsection{Integral Representation}
\label{app:psr_representation}

As reviewed in Section~\ref{sec:bg_psr}, binary strictly proper scoring rules
admit an integral representation in terms of elementary threshold
losses~\citep{schervish1989general,buja2005loss}:
\begin{equation}
    \phi(s,y)
    =
    \int_0^1 w(t)L_t(s,y)\,\mathrm{d}t,
\end{equation}
where $w(t)\geq0$ is the weight function associated with the scoring rule and
\begin{equation}
    L_t(s,y)
    :=
    y(1-t)\mathbf{1}\{s<t\}
    +
    (1-y)t\mathbf{1}\{s\geq t\}.
\end{equation}

The key point for our setting is that $L_t(S,Y)$ depends on the hidden score
$S$ only through the comparison $\mathbf{1}\{S\geq t\}$. Thus, after querying
at threshold $t$ and observing
$B=\mathbf{1}\{S\geq t\}$, the elementary loss can be written as
\begin{equation}
    L_t(S,Y)
    =
    Y(1-t)(1-B)+(1-Y)tB.
\end{equation}
The continuous score itself is not needed.

\subsubsection{Unbiasedness}
\label{app:psr_unbiasedness}

Let $T$ be drawn from a density $g$ on $(0,1)$ such that $g(t)>0$ wherever
$w(t)L_t(S,Y)$ can be nonzero. The pointwise one-bit estimator is
\begin{equation}
    \widehat{\phi}_{\mathrm{1bit}}(B,Y,T)
    :=
    \frac{w(T)}{g(T)}
    \left[
        Y(1-T)(1-B)
        +
        (1-Y)TB
    \right],
\end{equation}
where $B=\mathbf{1}\{S\geq T\}$.

\begin{lemma}[Pointwise unbiasedness]
For every fixed $(S,Y)$,
\begin{equation}
    \mathbb{E}_{T\sim g}
    \left[
        \widehat{\phi}_{\mathrm{1bit}}(B,Y,T)
        \mid S,Y
    \right]
    =
    \phi(S,Y).
\end{equation}
\end{lemma}

\begin{proof}
Taking expectation over the sampled threshold gives
\begin{align*}
\mathbb{E}_{T\sim g}
\left[
    \widehat{\phi}_{\mathrm{1bit}}
    \mid S,Y
\right]
&=
\int_0^1
\frac{w(t)}{g(t)}
L_t(S,Y)
g(t)\,\mathrm{d}t
\\
&=
\int_0^1
w(t)L_t(S,Y)\,\mathrm{d}t
\\
&=
\phi(S,Y).
\end{align*}
\end{proof}

Now draw independent thresholds $T_i\sim g$ and define
\[
    \widehat{R}_{\phi,\mathrm{1bit}}
    :=
    \frac{1}{N}
    \sum_{i=1}^N
    \widehat{\phi}_{\mathrm{1bit}}(B_i,Y_i,T_i).
\]
Taking expectation also over the sampled data,
\begin{equation}
    \mathbb{E}
    \left[
        \widehat{R}_{\phi,\mathrm{1bit}}
    \right]
    =
    \mathbb{E}[\phi(S,Y)]
    =
    R_\phi(f).
\end{equation}
Hence, the one-bit estimator is unbiased for the population proper-score risk.

\subsubsection{Proof of Theorem~\ref{thm:psr_variance}}
\label{app:proof_theorem2}

The one-bit estimator is unbiased, but the random threshold introduces
additional variance. We now derive this cost exactly.

\begin{proof}
Let
\[
    Z_i
    :=
    \widehat{\phi}_{\mathrm{1bit}}(B_i,Y_i,T_i).
\]
Since the observations and thresholds are independent across $i$,
\[
    \operatorname{Var}
    \left(
        \widehat{R}_{\phi,\mathrm{1bit}}
    \right)
    =
    \frac{1}{N}\operatorname{Var}(Z_i).
\]

Applying the law of total variance conditional on $(S_i,Y_i)$,
\begin{align*}
\operatorname{Var}(Z_i)
&=
\operatorname{Var}
\left(
    \mathbb{E}[Z_i\mid S_i,Y_i]
\right)+
\mathbb{E}
\left[
    \operatorname{Var}(Z_i\mid S_i,Y_i)
\right].
\end{align*}
By pointwise unbiasedness,
\[
    \mathbb{E}[Z_i\mid S_i,Y_i]
    =
    \phi(S_i,Y_i),
\]
so the first term is
\[
    \operatorname{Var}(\phi(S,Y)).
\]
The second term is, by definition,
\[
    \mathbb{E}\left[
        V_{\phi,g}(S,Y)
    \right]
    =
    J_{\phi,g}.
\]
Therefore,
\[
    \operatorname{Var}(Z_i)
    =
    \operatorname{Var}(\phi(S,Y))
    +
    J_{\phi,g}.
\]
Dividing by $N$ gives
\begin{align*}
    \operatorname{Var}
    \left(
        \widehat{R}_{\phi,\mathrm{1bit}}
    \right)
    &=
    \frac{\operatorname{Var}(\phi(S,Y))}{N}
    +
    \frac{J_{\phi,g}}{N}
    \\
    &=
    \operatorname{Var}
    \left(
        \widehat{R}_{\phi,\mathrm{full}}
    \right)
    +
    \frac{J_{\phi,g}}{N}.
\end{align*}
\end{proof}

\subsubsection{Variance Penalties for Brier Score and Log-Loss}
\label{app:psr_analytical_variance}

For a fixed score-label pair $(s,y)$, the conditional variance introduced by
threshold sampling is
\begin{equation}
\label{eq:app_psr_conditional_variance}
    V_{\phi,g}(s,y)
    =
    \int_0^1
    \frac{w(t)^2}{g(t)}
    L_t(s,y)^2
    \,\mathrm{d}t
    -
    \phi(s,y)^2.
\end{equation}
We now evaluate this quantity for uniform threshold sampling,
$g(t)=1$.

\paragraph{Brier score.}
For the Brier score, $w(t)=2$. Taking $y=1$,
\[
    L_t(s,1)^2
    =
    (1-t)^2\mathbf{1}\{s<t\}.
\]
Hence,
\begin{align*}
    \mathbb{E}_T
    \left[
        \widehat{\phi}_{\mathrm{1bit}}^2
        \mid S=s,Y=1
    \right]
    &=
    4\int_s^1(1-t)^2\,\mathrm{d}t
    \\
    &=
    \frac{4}{3}(1-s)^3.
\end{align*}
Since $\phi(s,1)=(1-s)^2$,
\begin{align}
    V_{\mathrm{Brier},1}(s,1)
    &=
    \frac{4}{3}(1-s)^3-(1-s)^4
    \nonumber\\
    &=
    (1-s)^3
    \left(
        s+\frac13
    \right).
\end{align}
This is maximized at $s=0$, where it equals $1/3$. By symmetry,
\begin{equation}
    V_{\mathrm{Brier},1}(s,y)
    \leq
    \frac13
\end{equation}
for every $s\in(0,1)$ and $y\in\{0,1\}$. Consequently,
\[
    J_{\mathrm{Brier},1}
    \leq
    \frac13
\]
for any distribution of predictions and labels.

\paragraph{Log-loss.}
For Log-loss,
\[
    w(t)
    =
    \frac{1}{t(1-t)}.
\]
For $y=1$,
\begin{align*}
    \mathbb{E}_T
    \left[
        \widehat{\phi}_{\mathrm{1bit}}^2
        \mid S=s,Y=1
    \right]
    &=
    \int_s^1
    \frac{1}{t^2}\,\mathrm{d}t
    \\
    &=
    \frac{1}{s}-1.
\end{align*}
Since $\phi(s,1)=-\log s$,
\begin{equation}
\label{eq:log_loss_variance}
    V_{\mathrm{Log},1}(s,1)
    =
    \frac{1}{s}
    -
    1
    -
    \log^2 s.
\end{equation}
By symmetry,
\begin{equation}
    V_{\mathrm{Log},1}(s,0)
    =
    \frac{1}{1-s}
    -
    1
    -
    \log^2(1-s).
\end{equation}

Unlike for the Brier score, the conditional variance is not uniformly bounded.
In particular,
\[
    V_{\mathrm{Log},1}(s,1)\to\infty
    \qquad\text{as }s\to0,
\]
corresponding to an increasingly confident error when $Y=1$. The symmetric
divergence occurs as $s\to1$ when $Y=0$.

For a perfectly calibrated predictor,
$\mathbb{P}(Y=1\mid S=s)=s$. Conditioning on $S=s$ therefore gives
\begin{align*}
\mathbb{E}
\left[
    V_{\mathrm{Log},1}(s,Y)
    \mid S=s
\right]
&=
s
\left(
    \frac{1}{s}-1-\log^2s
\right)
\\
&\quad+
(1-s)
\left(
    \frac{1}{1-s}-1-\log^2(1-s)
\right)
\\
&=
1
-
s\log^2s
-
(1-s)\log^2(1-s)
\\
&\leq
1.
\end{align*}
Averaging over $S$ yields
\begin{equation}
    J_{\mathrm{Log},1}
    \leq
    1
\end{equation}
for perfectly calibrated predictors.

\subsection{Recalibration}

\subsubsection{Temperature Scaling}
\label{app:ts_theory}

We provide additional details for the centered one-bit Temperature Scaling
objective and prove Proposition~\ref{prop:centered_ts}. Recall that
\[
    q_\alpha(s)
    =
    \sigma\!\bigl(\alpha\sigma^{-1}(s)\bigr),
\]
and let
\[
    R(\alpha)
    :=
    \mathbb{E}[\ell_\alpha(S,Y)]
\]
denote the population Log-loss after Temperature Scaling.

\paragraph{Behavior of the direct objective near $\alpha=0$.}
The direct one-bit estimator in
Equation~\ref{eq:ts_onebit_direct} is unbiased for every fixed
$\alpha>0$. However, its finite-sample behavior near $\alpha=0$ differs from
that of the population objective.

For any fixed observation $(B,Y,T)$ with $T\in(0,1)$,
\begin{equation*}
    \left|
        \widehat{\ell}^{\rm 1bit}_{\alpha}(B,Y,T)
    \right|
    \leq
    \frac{\alpha}{g(T)T(1-T)},
\end{equation*}
because the elementary loss lies in $[0,1]$. Hence,
\begin{equation*}
    \widehat{\ell}^{\rm 1bit}_{\alpha}(B,Y,T)
    \longrightarrow 0
    \qquad\text{as }\alpha\to0.
\end{equation*}
Therefore, for any fixed finite dataset,
\begin{equation*}
    \widehat{R}_{\rm 1bit}(\alpha)
    :=
    \frac{1}{N}
    \sum_{i=1}^N
    \widehat{\ell}^{\rm 1bit}_{\alpha}(B_i,Y_i,T_i)
    \longrightarrow 0.
\end{equation*}

The population objective has a different limit. Since
$q_\alpha(S)\to1/2$ as $\alpha\to0$,
\begin{equation*}
    R(\alpha)
    \longrightarrow
    \log 2.
\end{equation*}
Thus, pointwise unbiasedness for each fixed $\alpha>0$ does not imply that the
finite-sample objective has the correct boundary behavior as
$\alpha\to0$. In particular, the limit in $\alpha$ cannot in general be
swapped with the expectation over the random threshold.

\paragraph{Behavior of the centered objective near $\alpha=0$.}
For the centered estimator in
Equation~\ref{eq:ts_onebit_centered},
\begin{equation*}
    \widehat{\ell}^{\rm 1bit,cent}_{\alpha}(B,Y,T)
    -
    \log 2
    =
    \frac{
        \alpha\bigl(q_\alpha(T)-Y\bigr)
    }{
        g(T)T(1-T)
    }
    \left(
        B-\mathbf{1}\{T\leq1/2\}
    \right).
\end{equation*}
Since
\[
    |q_\alpha(T)-Y|\leq1
    \qquad\text{and}\qquad
    \left|
        B-\mathbf{1}\{T\leq1/2\}
    \right|
    \leq1,
\]
we obtain
\begin{equation*}
    \left|
        \widehat{\ell}^{\rm 1bit,cent}_{\alpha}(B,Y,T)
        -
        \log 2
    \right|
    \leq
    \frac{\alpha}{g(T)T(1-T)}.
\end{equation*}
Hence, for every fixed observation,
\begin{equation*}
    \widehat{\ell}^{\rm 1bit,cent}_{\alpha}(B,Y,T)
    \longrightarrow
    \log 2
    \qquad\text{as }\alpha\to0.
\end{equation*}
Averaging over any finite dataset therefore gives
\begin{equation*}
    \widehat{R}^{\rm cent}_{\rm 1bit}(\alpha)
    \longrightarrow
    \log 2.
\end{equation*}
Thus, unlike the direct estimator, the centered empirical objective has the
same limit at $\alpha=0$ as the population TS objective.

\paragraph{Proof of Proposition~\ref{prop:centered_ts}.}
Because $q_\alpha$ is strictly increasing,
\begin{equation*}
    \mathbf{1}\{
        q_\alpha(S)\geq q_\alpha(t)
    \}
    =
    \mathbf{1}\{S\geq t\}.
\end{equation*}
Also, since $q_\alpha(1/2)=1/2$,
\begin{equation*}
    \mathbf{1}\{
        1/2\geq q_\alpha(t)
    \}
    =
    \mathbf{1}\{t\leq1/2\}.
\end{equation*}

Using the definition of the elementary loss,
\begin{align*}
    &L_{q_\alpha(t)}(q_\alpha(S),Y)
    -
    L_{q_\alpha(t)}(1/2,Y) =
    \bigl(q_\alpha(t)-Y\bigr)
    \left(
        \mathbf{1}\{S\geq t\}
        -
        \mathbf{1}\{t\leq1/2\}
    \right).
\end{align*}
Therefore, conditioning on $(S,Y)$ and averaging over $T\sim g$,
\begin{align*}
    \mathbb{E}_{T\sim g}
    \left[
        \widehat{\ell}^{\rm 1bit,cent}_{\alpha}(B,Y,T)
        \mid S,Y
    \right]&=
    \log 2
    +
    \int_0^1
    \frac{\alpha}{t(1-t)}
    \Big[
        L_{q_\alpha(t)}(q_\alpha(S),Y)
        -
        L_{q_\alpha(t)}(1/2,Y)
    \Big]
    \,\mathrm{d}t
    \\
    &=
    \log 2
    +
    \ell_\alpha(S,Y)
    -
    \ell_\alpha(1/2,Y).
\end{align*}
Since
\[
    \ell_\alpha(1/2,Y)=\log 2,
\]
we obtain
\begin{equation*}
    \mathbb{E}_{T\sim g}
    \left[
        \widehat{\ell}^{\rm 1bit,cent}_{\alpha}(B,Y,T)
        \mid S,Y
    \right]
    =
    \ell_\alpha(S,Y).
\end{equation*}
This proves the first claim of
Proposition~\ref{prop:centered_ts}. The second claim is the boundary result
derived above.

For the third claim, let
$(X_i,Y_i)_{i=1}^N\sim\mathcal{D}^N$ and draw
$T_1,\ldots,T_N$ independently from $g$. Then
\begin{align*}
    \mathbb{E}
    \left[
        \widehat{R}^{\rm cent}_{\rm 1bit}(\alpha)
    \right]
    &=
    \frac{1}{N}
    \sum_{i=1}^N
    \mathbb{E}
    \left[
        \widehat{\ell}^{\rm 1bit,cent}_{\alpha}
        (B_i,Y_i,T_i)
    \right]
    \\
    &=
    \frac{1}{N}
    \sum_{i=1}^N
    \mathbb{E}
    \left[
        \ell_\alpha(S_i,Y_i)
    \right]
    \\
    &=
    R(\alpha).
\end{align*}
Thus, for every $\alpha>0$, the expected centered one-bit objective equals
the standard population TS objective. In particular, the two expected
objectives have the same population minimizers.

\subsubsection{General Monotone Recalibration}
\label{app:general_recalibration}

The centered one-bit construction is not specific to Temperature Scaling.
Consider a family of differentiable, strictly increasing bijections
$q_\theta:(0,1)\to(0,1)$, and define the recalibrated Log-loss
\begin{equation*}
    \ell_\theta(s,y)
    :=
    -y\log q_\theta(s)
    -(1-y)\log\!\bigl(1-q_\theta(s)\bigr).
\end{equation*}
Let $s_0\in(0,1)$ be a fixed anchor point. As before, draw
$T\sim g$ and observe
\[
    B=\mathbf{1}\{S\geq T\}.
\]

\begin{proposition}
\label{prop:general_recalibration}
Assume that $g(t)>0$ on $(0,1)$ and that the expression below is integrable.
Then, for every fixed $\theta$,
\begin{align}
    \widehat{\ell}^{\rm 1bit,cent}_\theta(B,Y,T)
    &:=
    \ell_\theta(s_0,Y)
    +
    \frac{
        q_\theta'(T)\bigl(q_\theta(T)-Y\bigr)
    }{
        g(T)q_\theta(T)(1-q_\theta(T))
    }
    \left(
        B-\mathbf{1}\{T\leq s_0\}
    \right)
    \label{eq:general_recalibration}
\end{align}
is conditionally unbiased:
\begin{equation*}
    \mathbb{E}_{T\sim g}
    \left[
        \widehat{\ell}^{\rm 1bit,cent}_\theta(B,Y,T)
        \mid S,Y
    \right]
    =
    \ell_\theta(S,Y).
\end{equation*}
Consequently,
\begin{equation*}
    \mathbb{E}
    \left[
        \widehat{\ell}^{\rm 1bit,cent}_\theta(B,Y,T)
    \right]
    =
    \mathbb{E}
    \left[
        \ell_\theta(S,Y)
    \right],
\end{equation*}
so the expected one-bit and probability-access objectives have the same
population minimizers over $\theta$.
\end{proposition}

\begin{proof}
Since $q_\theta$ is strictly increasing,
\begin{equation*}
    \mathbf{1}\{
        q_\theta(S)\geq q_\theta(t)
    \}
    =
    \mathbf{1}\{S\geq t\},
\end{equation*}
and similarly,
\begin{equation*}
    \mathbf{1}\{
        q_\theta(s_0)\geq q_\theta(t)
    \}
    =
    \mathbf{1}\{s_0\geq t\}.
\end{equation*}
Using the elementary Log-loss representation,
\begin{align*}
    &L_{q_\theta(t)}(q_\theta(S),Y)
    -
    L_{q_\theta(t)}(q_\theta(s_0),Y)
    =
    \bigl(q_\theta(t)-Y\bigr)
    \left(
        \mathbf{1}\{S\geq t\}
        -
        \mathbf{1}\{t\leq s_0\}
    \right).
\end{align*}

Applying the change of variables $u=q_\theta(t)$ to the Log-loss integral
therefore gives
\begin{align*}
    &\ell_\theta(S,Y)
    -
    \ell_\theta(s_0,Y) =
    \int_0^1
    \frac{
        q_\theta'(t)
    }{
        q_\theta(t)(1-q_\theta(t))
    }
    \bigl(q_\theta(t)-Y\bigr)
    \left(
        \mathbf{1}\{S\geq t\}
        -
        \mathbf{1}\{t\leq s_0\}
    \right)
    \,\mathrm{d}t.
\end{align*}
Sampling $T\sim g$ and applying importance weighting gives
\begin{align*}
    \mathbb{E}_{T\sim g}
    \left[
        \widehat{\ell}^{\rm 1bit,cent}_\theta(B,Y,T)
        \mid S,Y
    \right]&=
    \ell_\theta(s_0,Y)
    +
    \ell_\theta(S,Y)
    -
    \ell_\theta(s_0,Y)
    \\
    &=
    \ell_\theta(S,Y).
\end{align*}
Averaging over $(S,Y)$ proves the population statement.
\end{proof}

Temperature Scaling is recovered by taking
$q_\theta=q_\alpha$ and $s_0=1/2$. In this case,
$q_\alpha(1/2)=1/2$ and
\[
    \ell_\alpha(1/2,Y)=\log 2.
\]
Moreover,
\[
    q_\alpha'(t)
    =
    \frac{
        \alpha q_\alpha(t)(1-q_\alpha(t))
    }{
        t(1-t)
    },
\]
so Equation~\ref{eq:general_recalibration} reduces exactly to the centered
one-bit TS estimator in Equation~\ref{eq:ts_onebit_centered}.

\newpage
\section{Additional Synthetic Experiments}
\label{app:synthetic}

\subsection{Synthetic Data-Generating Process}
\label{app:synthetic_dgp}

We use the same one-dimensional data-generating process throughout the
synthetic calibration experiments. We draw
$U\sim\mathrm{Uniform}[-1,1]$ and set
\begin{equation*}
    S=\sigma(AU),
    \qquad
    Y\mid U\sim
    \mathrm{Bernoulli}\!\left(\sigma(\kappa U)\right),
    \qquad
    \kappa=2.
\end{equation*}
The parameter $A$ controls how extreme the predicted scores are. When
$A=\kappa=2$, $S$ equals the true event probability. For $A>\kappa$, the
scores become increasingly overconfident.

Since $S=\sigma(AU)$ is strictly increasing in $U$,
$U=\sigma^{-1}(S)/A$. The calibration function is therefore
\begin{equation}
\label{eq:synthetic_calibration_function}
    c_f(s)
    =
    \mathbb{P}(Y=1\mid S=s)
    =
    \sigma\!\left(
        \frac{\kappa}{A}\sigma^{-1}(s)
    \right).
\end{equation}
In particular, $c_f(s)=s$ when $A=\kappa$, so the predictor is perfectly
calibrated.

The population TCE is available in closed form. By symmetry,
\begin{align*}
    \mathrm{TCE}
    &=
    \frac{1}{2}
    \int_{-1}^{1}
    \left|
        \sigma(\kappa u)-\sigma(Au)
    \right|
    \,\mathrm{d}u \\
    &=
    \left|
        \frac{\log(1+e^A)-\log 2}{A}
        -
        \frac{\log(1+e^\kappa)-\log 2}{\kappa}
    \right|.
\end{align*}
We use this exact value as the population target in all MSE calculations.
For $A=\kappa=2$, $\mathrm{TCE}=0$.

\subsection{Calibration Error}
\label{app:synthetic_curves}

\paragraph{Monte Carlo protocol.}
We use
$A\in\{2,4,6,8,10,12\}$ and
\begin{equation*}
    N\in
    \{2^{10},2^{11},2^{12},2^{13},
      2^{14},2^{15},2^{16},2^{17}\}.
\end{equation*}
For $M$ bins, we use equal-width boundaries $t_m=m/M$ and midpoints
$a_m=(t_{m-1}+t_m)/2$.

For the one-bit estimator, observations are randomly shuffled and partitioned
into $M$ disjoint groups
$\mathcal{I}_0,\ldots,\mathcal{I}_{M-1}$ whose sizes differ by at most one.
Samples in $\mathcal{I}_m$ are assigned to threshold $t_m$, matching the
construction in Section~\ref{sec:onebit_calib}. The boundary values
$Q_0\equiv1$ and $Q_1\equiv0$ are deterministic and require no model query.

The probability-access baseline is the standard equal-width empirical
$L_1$ ECE computed from the observed scores $S_i$. Within each trial, both
estimators use the same sampled dataset.

For an estimator $\widehat{\mathrm{CE}}$, we report
\begin{equation*}
    \mathrm{MSE}
    =
    \frac{1}{R}
    \sum_{r=1}^R
    \left(
        \widehat{\mathrm{CE}}^{(r)}
        -
        \mathrm{TCE}
    \right)^2.
\end{equation*}
We use $R=500$ independent trials when selecting the number of bins and the convergence experiments.

\paragraph{Empirical choice of the number of bins.}
For each $(A,N)$, we evaluate
$M\in\{2,\ldots,40\}$ and let $M^\star(A,N)$ denote the value with the
smallest mean MSE over the $500$ trials. Exact ties are broken in favor of the
smaller $M$.

For each fixed $A$, Figure~\ref{fig:optimal_bins_multiple_A} shows how
$M^\star(A,N)$ changes with $N$. We fit
$M^\star(A,N)\propto N^\beta$ by ordinary least squares in log--log space
using all eight sample sizes.

\begin{figure*}[t]
    \centering
    \includegraphics[width=0.8\textwidth]
    {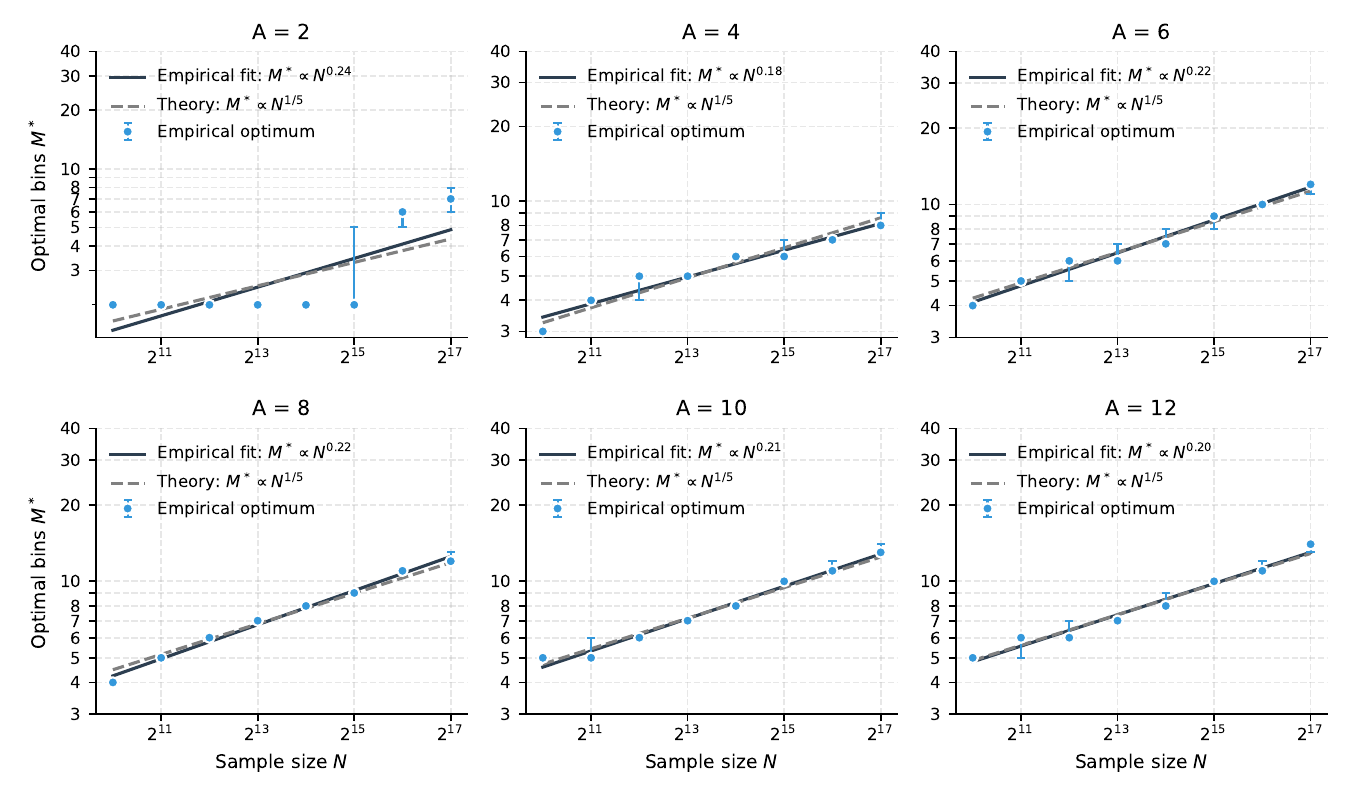}
    \caption{
    \textbf{Empirical MSE-minimizing number of bins.}
    For each $(A,N)$, $M^\star(A,N)$ minimizes the Monte Carlo MSE over
    $M\in\{2,\ldots,40\}$. The solid line shows the fitted power law and the
    dashed line the theoretical scaling $M\propto N^{1/5}$.
    The fitted exponents are $0.24$, $0.18$, $0.22$, $0.22$, $0.21$, and
    $0.20$ for $A=2,4,6,8,10,12$, respectively.
    Error bars show $95\%$ bootstrap intervals for $M^\star$.
    }
    \label{fig:optimal_bins_multiple_A}
\end{figure*}

The fitted exponents are close to the theoretical value $1/5$ across the six
settings. Error bars are $95\%$ bootstrap intervals obtained from $1000$
bootstrap repetitions. In each repetition, the $500$ Monte Carlo trials are
resampled with replacement, the mean MSE is recomputed for every $M$, and a
new $M^\star$ is selected. The intervals are the $2.5$th and $97.5$th
percentiles of the resulting bootstrap distribution.

\paragraph{Full MSE curves.}
Figure~\ref{fig:synthetic_mse_vs_m} shows the MSE curves underlying the
selection of $M^\star$. The four displayed sample sizes illustrate how the
preferred number of bins increases with $N$ across the different values of
$A$.

\begin{figure*}[t]
    \centering
    \includegraphics[width=0.8\textwidth]
    {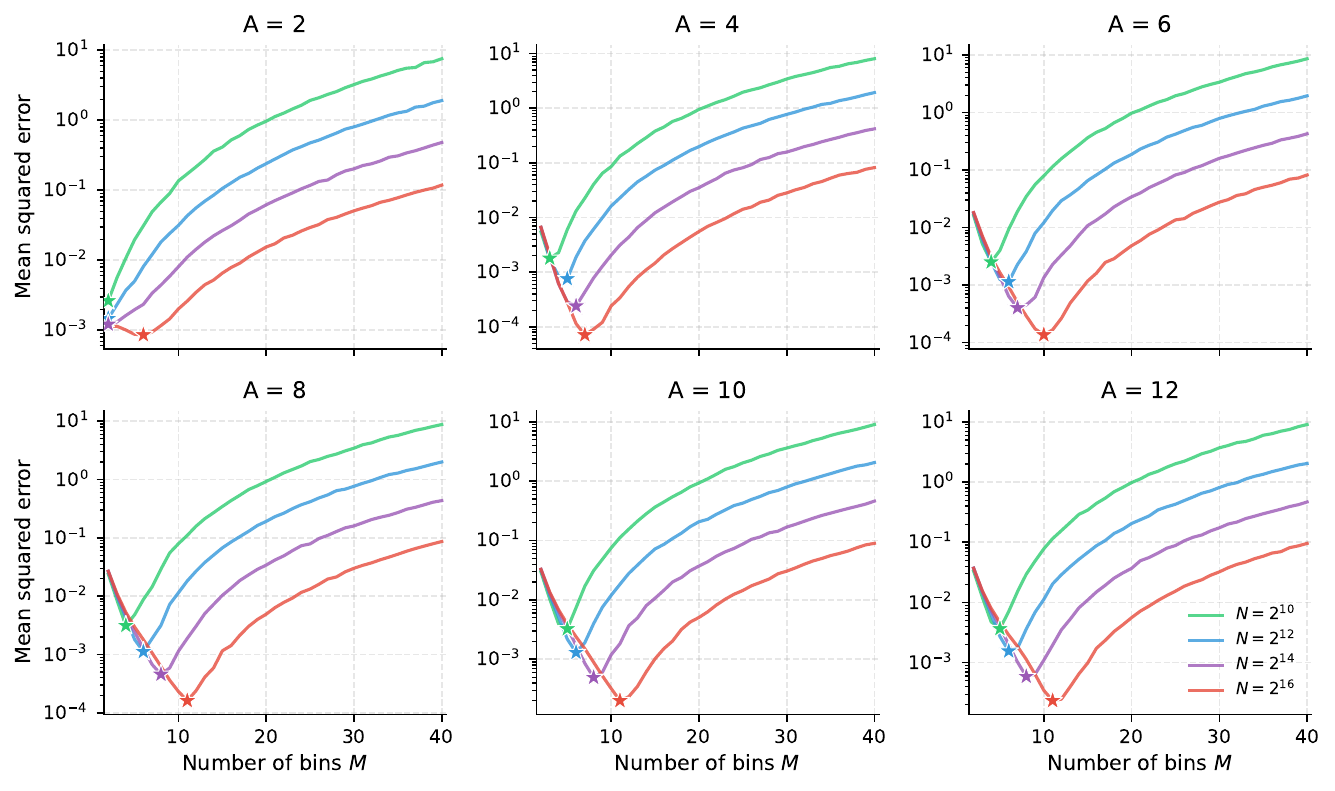}
    \caption{
    \textbf{MSE as a function of the number of bins.}
    Mean squared error of the one-bit calibration estimator over $500$
    independent trials for
    $A\in\{2,4,6,8,10,12\}$ and four representative sample sizes.
    Each curve shows the MSE over $M\in\{2,\ldots,40\}$, and stars mark the
    empirical minimizers $M^\star(A,N)$. As $N$ increases, the minimizing
    number of bins shifts toward larger values.
    }
    \label{fig:synthetic_mse_vs_m}
\end{figure*}

\paragraph{MSE convergence.}
In the main text, we study the calibrated setting $A=2$ and compare the
empirical MSE convergence with the theoretical rates $N^{-2/5}$ for the
one-bit estimator and $N^{-2/3}$ for probability-access ECE. The reference
lines in the figure have these fixed slopes; their vertical offsets are fitted
only for visualization.

We additionally study the convergence of the one-bit estimator across all
values of $A$. At each $(A,N)$, we use the empirically selected
$M^\star(A,N)$. The results are shown in
Figure~\ref{fig:mse_vs_N_multiple_A}.

\begin{figure*}[t]
    \centering
    \includegraphics[width=0.8\textwidth]
    {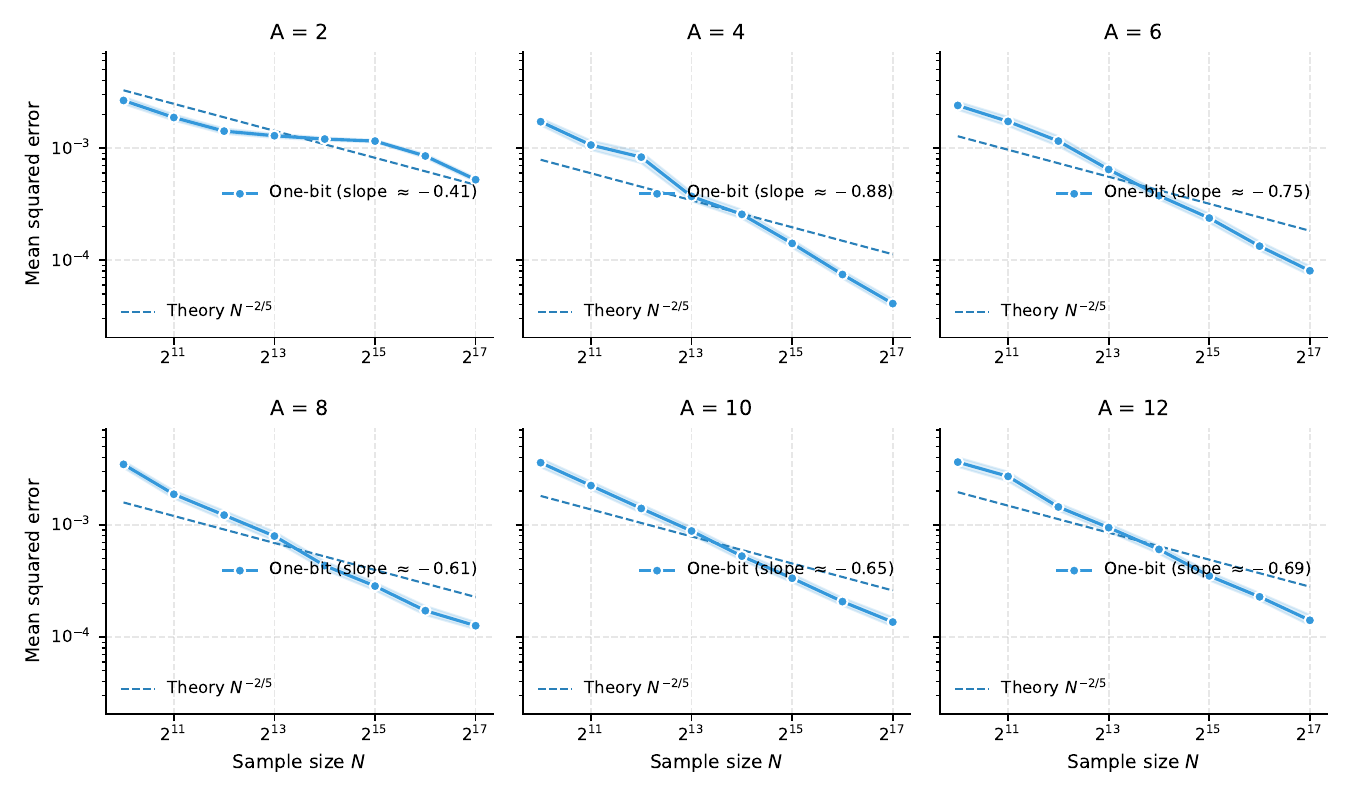}
    \caption{
    \textbf{MSE convergence of the one-bit calibration estimator.}
    MSE is computed over $500$ independent trials using the empirically
    selected $M^\star(A,N)$ at each sample size.
    The dashed line shows the theoretical upper-bound rate $N^{-2/5}$.
    The empirical log--log slopes are approximately
    $-0.41$, $-0.88$, $-0.75$, $-0.61$, $-0.65$, and $-0.69$ for
    $A=2,4,6,8,10,12$, respectively.
    }
    \label{fig:mse_vs_N_multiple_A}
\end{figure*}

For $A=2$, the empirical slope is approximately $-0.41$, close to the
$-2/5$ exponent of the theoretical upper bound. In the miscalibrated
settings, the MSE decreases faster over the sample sizes considered. This does
not contradict Theorem~\ref{thm:onebit_consistency}, which gives an upper
bound rather than an exact convergence rate.

\clearpage
\subsection{Proper Scoring Rules}
\label{app:psr_experiments}

\begin{figure*}[b]
    \centering
    \includegraphics[width=\textwidth]
    {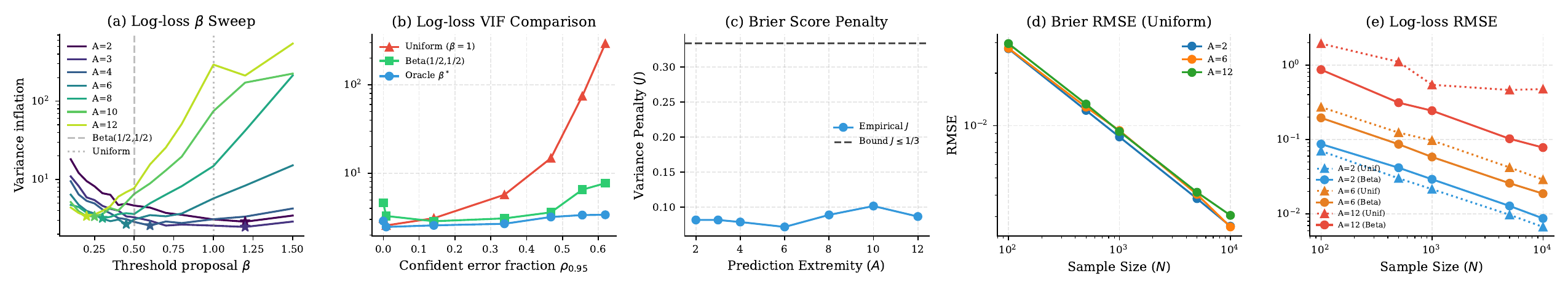}
    \caption{
    \textbf{Additional synthetic proper-score experiments.}
    \textbf{(a)} Variance inflation of the one-bit Log-loss estimator under
    $T\sim\mathrm{Beta}(\beta,\beta)$; stars mark the variance-minimizing
    $\beta^\star$ on the tested grid, while the vertical lines indicate
    $\mathrm{Beta}(1/2,1/2)$ and Uniform sampling.
    \textbf{(b)} Log-loss variance inflation versus the fraction of confident
    errors for Uniform, fixed $\mathrm{Beta}(1/2,1/2)$, and oracle threshold
    proposals.
    \textbf{(c)} Empirical $J_{\mathrm{Brier},1}$ under Uniform thresholds,
    compared with the theoretical bound
    $J_{\mathrm{Brier},1}\leq1/3$.
    \textbf{(d)} Brier-score RMSE between one-bit and probability-access risk
    estimates as the sample size increases.
    \textbf{(e)} Log-loss RMSE for Uniform and
    $\mathrm{Beta}(1/2,1/2)$ thresholds.
    Panels (d)--(e) use $A\in\{2,6,12\}$.
    All Monte Carlo results use $R=500$ independent trials.
    }
    \label{fig:psr_synthetic_appendix}
\end{figure*}

We provide additional experiments for the one-bit proper-score estimators.
We first study the choice of threshold distribution, then examine the
Brier-score variance bound and finite-sample estimation. All Monte Carlo
results use $R=500$ independent trials.

\paragraph{Threshold sampling for Log-loss.}
We use the synthetic family from Appendix~\ref{app:synthetic_dgp} with
$\kappa=2$ and $A\in\{2,3,4,6,8,10,12\}$. We consider
\[
    T\sim\mathrm{Beta}(\beta,\beta),
\]
with
\[
    \beta\in
    \{0.1,0.15,0.2,0.25,0.3,0.35,0.4,0.45,
      0.5,0.6,0.7,0.8,1,1.2,1.5\}.
\]
For each $\beta$, we measure variance inflation as
\[
    \frac{
        \operatorname{Var}(\widehat R_{\mathrm{Log},\mathrm{1bit}})
    }{
        \operatorname{Var}(\widehat R_{\mathrm{Log},\mathrm{full}})
    }.
\]
We denote by $\beta^\star$ the value on this grid that minimizes the empirical
one-bit variance. This is an oracle diagnostic, since it is selected
retrospectively after comparing all candidate proposals.

Figure~\ref{fig:psr_synthetic_appendix}(a) shows that the preferred proposal
changes with the error regime. The minimizing value is $\beta^\star=1.2$ for
$A=2$ and $A=3$, and decreases as the scores become more extreme, reaching
$\beta^\star=0.2$ at $A=12$. Thus, as confident errors become more extreme,
the variance-minimizing proposal places increasingly more mass near the
endpoints.

Figure~\ref{fig:psr_synthetic_appendix}(b) compares Uniform sampling, the fixed
$\mathrm{Beta}(1/2,1/2)$ proposal, and the oracle choice. Uniform sampling
works well in mild regimes but becomes much less efficient as confident errors
become more common. 

The oracle $\beta^\star$ illustrates how the best proposal depends on the
score--label distribution, but it requires information that is unavailable
under restricted access. We therefore use the fixed
$\mathrm{Beta}(1/2,1/2)$ proposal as a simple endpoint-focused alternative in
the remaining experiments. It requires no access to the hidden scores and
performs well across the regimes considered here.

\paragraph{Brier score.}
For the Brier score, Uniform threshold sampling satisfies
$J_{\mathrm{Brier},1}\leq1/3$.
Figure~\ref{fig:psr_synthetic_appendix}(c) shows the empirical value of
$J_{\mathrm{Brier},1}$ across the same values of $A$.

The empirical penalty remains well below the theoretical bound. Unlike Log-loss, making the scores
more extreme does not lead to a large increase in the additional variance from
one-bit threshold sampling.

\paragraph{Finite-sample estimation.}
We next compare the one-bit and probability-access risk estimates as the
sample size increases. Both estimators use the same sampled dataset in each
trial. We report
\[
    \mathrm{RMSE}
    =
    \sqrt{
        \frac{1}{R}
        \sum_{r=1}^{R}
        \left(
        \widehat R_{\mathrm{1bit}}^{(r)}
        -
        \widehat R_{\mathrm{full}}^{(r)}
        \right)^2
    },
    \qquad R=500.
\]
We use $N\in\{100,500,1000,5000,10000\}$ and show three representative
settings, $A\in\{2,6,12\}$.

For the Brier score, the RMSE decreases steadily with $N$ and remains similar
across the three settings. For Log-loss, Uniform sampling also performs well
when $A=2$, but its error becomes much larger for extreme
scores. The fixed $\mathrm{Beta}(1/2,1/2)$ proposal is more stable in this
regime. For example, at $A=12$, replacing Uniform sampling by
$\mathrm{Beta}(1/2,1/2)$ reduces the RMSE across $N$.

\clearpage
\subsection{Temperature Scaling}
\label{app:ts_experiments}

We provide additional synthetic experiments for the Temperature Scaling
results in Section~\ref{sec:exp_temperature_scaling}. We use the same synthetic
family with $\kappa=2$, for which the population-optimal inverse temperature is
\[
    \alpha^\star=\frac{2}{A}.
\]
We vary
\[
    N_{\mathrm{cal}}
    \in
    \{128,256,512,1024,2048,4096,8192\}
\]
and repeat each setting over $500$ independent trials. We compare
probability-access TS with centered one-bit TS using Uniform and
$\mathrm{Beta}(1/2,1/2)$ threshold proposals.

\paragraph{Finite-sample behavior.}
Figure~\ref{fig:ts_synthetic_appendix} reports the excess population Log-loss,
$R(\widehat{\alpha})-R(\alpha^\star)$, and the estimation error
$|\widehat{\alpha}-\alpha^\star|$. Both decrease as the calibration set grows.
For mild score distributions, Uniform and $\mathrm{Beta}(1/2,1/2)$ perform
similarly. As the scores become more extreme, Uniform sampling becomes less
stable, while the endpoint-focused Beta proposal remains more accurate.

\begin{figure*}[h]
    \centering
    \includegraphics[width=\textwidth]
    {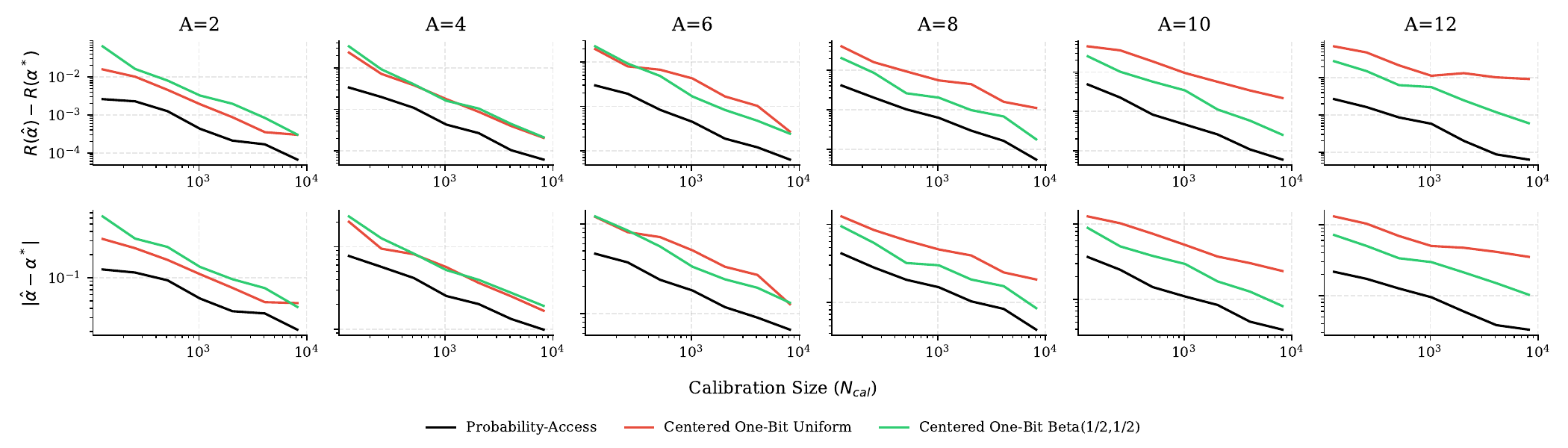}
    \caption{
    \textbf{Finite-sample Temperature Scaling on the synthetic family.}
    Top row: excess population Log-loss
    $R(\widehat{\alpha})-R(\alpha^\star)$ as the calibration-set size grows.
    Bottom row: mean absolute error in the fitted inverse temperature,
    $|\widehat{\alpha}-\alpha^\star|$.
    We compare probability-access TS with centered one-bit TS using Uniform
    and $\mathrm{Beta}(1/2,1/2)$ threshold proposals. Endpoint-focused
    sampling becomes increasingly helpful as the scores become more extreme.
    }
    \label{fig:ts_synthetic_appendix}
\end{figure*}

\begin{figure}[h]
    \centering
    \includegraphics[width=0.4\linewidth]{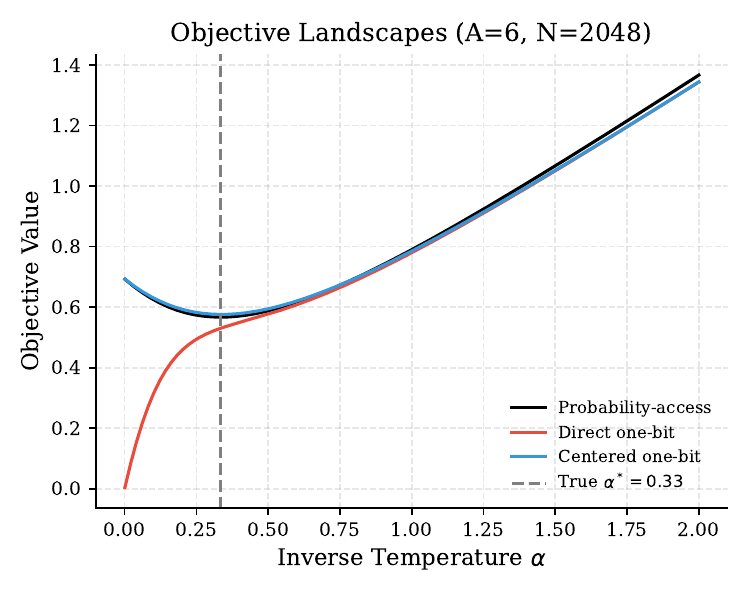}
    \caption{
    \textbf{Effect of centering the one-bit TS objective.}
    Empirical objectives for probability-access TS, the direct one-bit
    construction, and the centered one-bit construction at $A=6$ and
    $N_{\mathrm{cal}}=2048$. The direct objective approaches zero near
    $\alpha=0$, whereas the centered objective has the correct boundary
    behavior and retains a useful minimum near $\alpha^\star$.
    }
    \label{fig:ts_centering_diagnostic}
\end{figure}

\paragraph{Why centering matters.}
Figure~\ref{fig:ts_centering_diagnostic} illustrates the finite-sample behavior
described in Section~\ref{sec:onebit_temperature}. We use $A=6$,
$N_{\mathrm{cal}}=2048$, and Uniform thresholds. As $\alpha\to0$, the direct
one-bit empirical objective approaches zero, unlike the probability-access
objective, whose limit is $\log 2$. This can spuriously favor very small
values of $\alpha$. The centered objective has the correct limit and better
matches the probability-access objective around the population optimum
$\alpha^\star=2/6$.

\paragraph{Optimization details.}
For the synthetic experiments, we minimize the probability-access and centered
one-bit empirical objectives using
\texttt{scipy.optimize.minimize\_scalar} with the bounded method over
$\alpha\in[10^{-3},10]$. We allow at most $200$ iterations and set
$\texttt{xatol}=10^{-5}$.

\clearpage
\section{LLM Experimental Details}
\label{app:llm_details}

\subsection{Models and Benchmarks}
\label{app:models_and_benchmarks}

We evaluate four instruction-tuned LLMs:
\begin{itemize}
    \item Qwen 2.5-7B-Instruct
    (\texttt{Qwen/Qwen2.5-7B-Instruct})~\citep{qwen2025qwen25technicalreport};
    \item Llama 3.1-8B-Instruct
    (\texttt{meta-llama/Llama-3.1-8B-Instruct})~\citep{grattafiori2024llama3herdmodels};
    \item Mistral 7B-Instruct-v0.3
    (\texttt{mistralai/Mistral-7B-Instruct-v0.3})~\citep{jiang2023mistral7b};
    \item Gemma 2-9B-IT
    (\texttt{google/gemma-2-9b-it})~\citep{gemmateam2024gemma2improvingopen}.
\end{itemize}
The Hugging Face model identifiers are given above.

We consider two question-answering benchmarks from the Hugging Face Datasets
Hub. For BoolQ (\texttt{google/boolq})~\citep{clark2019boolq}, we use the
validation split, containing $3{,}270$ examples, and treat each example as a
binary true/false question. For MMLU (\texttt{cais/mmlu}, \texttt{all}
configuration)~\citep{mmlu}, we use the test split and convert each
multiple-choice question into four binary queries by pairing the question with
each candidate answer. This gives $56{,}168$ binary examples.

For MMLU, all binary queries derived from the same multiple-choice question are
kept together whenever the data are partitioned. This prevents examples from
the same original question from appearing in both calibration and test subsets.

\subsection{Prompt Design and Binary Score}
\label{app:llm_qa_protocol}

We use each LLM as a probabilistic binary classifier. Prior work has shown that
language-model probabilities can provide useful confidence signals when
questions are presented in suitable multiple-choice or true/false
formats~\citep{kadavath2022language}. We therefore require the model to choose
between two designated single-token answers, \texttt{A} and \texttt{B}.

For an input $X$, let $z_+(X)$ and $z_-(X)$ denote the logits of the tokens
assigned to the positive and negative semantic answers, respectively. We define
the hidden binary score as
\[
    S(X)
    :=
    \frac{e^{z_+(X)}}
         {e^{z_+(X)}+e^{z_-(X)}}
    =
    \sigma\!\bigl(z_+(X)-z_-(X)\bigr).
\]
Thus, $S(X)$ is the probability assigned to the positive answer after
normalizing over the two designated answer tokens. This defines the
probabilistic binary classifier whose calibration we evaluate in the LLM
experiments.

A possible issue with this construction is token bias: a model may naturally
prefer the token \texttt{A} over \texttt{B}, independently of their assigned
meaning. We therefore randomize the mapping between answer tokens and semantic
labels independently for each question:
\begin{itemize}
    \item with probability $1/2$, \texttt{A = True} and \texttt{B = False};
    \item with probability $1/2$, \texttt{A = False} and \texttt{B = True}.
\end{itemize}
Accordingly, $z_+(X)$ denotes the logit of whichever token is mapped to
\texttt{True}, and $z_-(X)$ the logit of the token mapped to \texttt{False}.
Randomizing this mapping prevents a systematic preference for one answer token
from consistently favoring either semantic class.

For BoolQ, we insert the question directly into the \texttt{\{statement\}}
slot. Under the standard mapping, the prompt is:
\newpage
\begin{lstlisting}[language=TeX, breaklines=true, frame=single]
Is the following statement true or false?

{statement}

A. True
B. False

Answer with A or B only.
Answer:
\end{lstlisting}

For MMLU, we combine the original question and one candidate answer into a
single statement:
\begin{lstlisting}[language=TeX, breaklines=true, frame=single]
Is the following statement true or false?

Question: {question} Answer: {choice}

A. True
B. False

Answer with A or B only.
Answer:
\end{lstlisting}
Under the reversed mapping, we simply swap the two options so that
\texttt{A = False} and \texttt{B = True}.

To quantify token bias, we keep the semantic statement fixed while swapping the
A/B mapping and measure the resulting change in the probability assigned to the
same semantic answer. Table~\ref{tab:mapping_bias} shows that this effect can be
substantial, particularly for Mistral and Llama.

\begin{table}[h]
\centering
\caption{
\textbf{Sensitivity to the A/B token mapping.}
Mean absolute change in the probability assigned to the same semantic answer
when the A/B labels are swapped.
}
\label{tab:mapping_bias}
\begin{tabular}{lc}
\toprule
\textbf{Model} & \textbf{Mean Abs. Probability Shift} \\
\midrule
Mistral v0.3-7B  & 23.2\% \\
Llama 3.1-8B     & 19.7\% \\
Qwen 2.5-7B      & 5.6\% \\
Gemma 2-9B       & 4.0\% \\
\bottomrule
\end{tabular}
\end{table}

Because this effect is large, we use the randomized mapping in all experiments
so that systematic token preferences average out across the dataset.

\subsection{Tokenization and Black-Box Binary Decoding}
\label{app:tokenization}

Because \texttt{logit\_bias} is applied to token IDs, we first identify the
tokens corresponding to the answer labels \texttt{A} and \texttt{B} in the
prompt context. We do this separately for each model, since tokenization can
differ across model families and depend on the preceding text. In all cases,
\texttt{A} and \texttt{B} correspond to distinct single tokens, as shown in
Table~\ref{tab:token_extraction}.

\begin{table}[h]
\centering
\caption{Contextual token IDs for the designated A/B answer labels.}
\label{tab:token_extraction}
\begin{tabular}{lrclr}
\toprule
\textbf{Model} & \textbf{ID A} & \textbf{Decoded A} &
\textbf{ID B} & \textbf{Decoded B} \\
\midrule
Qwen 2.5-7B      & 32     & \texttt{A} & 33     & \texttt{B} \\
Llama 3.1-8B     & 32     & \texttt{A} & 33     & \texttt{B} \\
Mistral v0.3-7B  & 29509  & \texttt{A} & 29528  & \texttt{B} \\
Gemma 2-9B       & 235280 & \texttt{A} & 235305 & \texttt{B} \\
\bottomrule
\end{tabular}
\end{table}

For the open-weight models, we can directly read the two answer-token logits
and compute the hidden reference score $S(X)$ defined in
Section~\ref{app:llm_qa_protocol}. We use this quantity only for the
probability-access reference and for validating the threshold-query mechanism;
it is never provided to the one-bit estimator.

In a black-box setting, these logits would not be available. Instead, we want
the generated token itself to reveal whether $S(X)$ lies above or below a
chosen threshold. To make the model effectively choose only between
\texttt{A} and \texttt{B}, we add the same large positive bias $M=50$ to both
answer tokens. Because the same constant is added to both logits, their
relative ordering is unchanged, while both become much more likely than any
other token.
We validate this mechanism using unrestricted greedy decoding on $4{,}000$
queries ($1{,}000$ per model), with thresholds ranging from $0.01$ to $0.99$.
Every generated token was either \texttt{A} or \texttt{B}, and the generated
answer agreed with the reference decision
\[
    Q_t(X)=\mathbf{1}\{S(X)\geq t\}
\]
on $100.0\%$ of queries across the four models.

For the large-scale experiments, we therefore simulate the one-bit response from the precomputed reference scores. All required LLM inference was performed on a single NVIDIA H100 GPU. Simulating the one-bit responses avoids unnecessary repeated GPU inference, reducing both computational cost and energy use. The validation above shows that, on the tested queries, this simulation matches the black-box-compatible \texttt{logit\_bias} mechanism.

\subsection{Implementation of Threshold Queries}
\label{app:threshold_implementation}

We now describe how a desired probability threshold $t$ is implemented using
\texttt{logit\_bias}~\citep{openai_logit_bias}. Since
\[
    S(X)=\sigma\!\bigl(z_+(X)-z_-(X)\bigr),
\]
we have
\[
    S(X)\geq t
    \quad\Longleftrightarrow\quad
    z_+(X)-z_-(X)\geq \sigma^{-1}(t).
\]

We add the common positive bias $M=50$ to both answer tokens, together with an
additional threshold-dependent bias $-\sigma^{-1}(t)$ on the positive token:
\[
    b_+(t)=M-\sigma^{-1}(t),
    \qquad
    b_-=M.
\]
Because the common term $M$ cancels when comparing the two tokens,
\[
    z_+(X)+b_+(t)\geq z_-(X)+b_-
    \quad\Longleftrightarrow\quad
    S(X)\geq t.
\]
Thus, the generated A/B answer directly encodes the threshold bit
\[
    Q_t(X)=\mathbf{1}\{S(X)\geq t\}.
\]

\begin{figure}[h]
    \centering
    \includegraphics[width=0.35\linewidth]
    {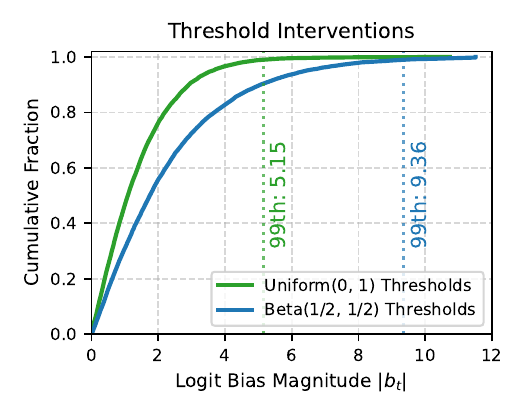}
    \caption{
    \textbf{Magnitude of threshold-dependent logit-bias interventions.}
    Empirical CDF of $|\sigma^{-1}(t)|$ for thresholds drawn from
    Uniform$(0,1)$ and Beta$(1/2,1/2)$. The Beta proposal samples thresholds
    closer to $0$ and $1$, leading to larger interventions more frequently.
    Dashed lines indicate the 99th percentiles.
    }
    \label{fig:bias_magnitude_cdf}
\end{figure}

The additional bias required to implement threshold $t$ has magnitude
$|\sigma^{-1}(t)|$. Figure~\ref{fig:bias_magnitude_cdf} shows its empirical
distribution for the two threshold proposals used in our experiments.
Beta$(1/2,1/2)$ places more mass near $0$ and $1$ and therefore requires
larger interventions more often: its 99th percentile is $9.36$, compared with
$5.15$ for Uniform$(0,1)$. The largest observed values are $11.51$ and
$10.74$, respectively. For the thresholds sampled in our experiments, all
resulting interventions remain within the $[-100,100]$ range supported by the
OpenAI \texttt{logit\_bias} interface.

\subsection{Auditing Baselines}
\label{app:auditing_baselines}

We compare our one-bit estimator against three restricted-access baselines.

\paragraph{Verbalized confidence.}
Following prior work on eliciting confidence directly from language
models~\citep{lin2022teachingmodelsexpressuncertainty,xiong2023can,tian-etal-2023-just},
we prompt the model to report both its binary answer and an explicit numerical
confidence:
\begin{lstlisting}[language=TeX, breaklines=true, frame=single]
You must ONLY output your answer (True or False) and your confidence percentage (e.g., 85%). Provide NO explanation.
\end{lstlisting}
The confidence is parsed as the first contiguous two- or three-digit sequence,
clipped to $[50,100]$, and rescaled to $c\in[0.5,1]$. If the reported answer is
True, we assign positive-class probability $c$; if it is False, we assign
$1-c$. The parsing-failure rate varies across model--dataset pairs, from nearly
$0\%$ for Qwen 2.5 and Gemma 2 to approximately $5$--$7\%$ for Llama 3.1 and
Mistral. Instances for which no confidence can be parsed are excluded from the
ECE computation for this baseline. Verbalized Confidence requires one model
call per example and generates at most 10 output tokens.

\paragraph{Sampling frequency.}
We estimate the positive-class probability by repeatedly sampling the binary
model output and using the empirical fraction of positive answers. Sampling is
performed with \texttt{temperature = 1.0}. With a budget of $K$ samples, this
requires $K$ model calls per example.

\paragraph{Carlini-style logit extraction.}
Following the repeated-bias strategy of \citet{carlini2024stealing}, we use
multiple \texttt{logit\_bias} interventions to localize the hidden logit
difference by binary search. The search is initialized over the interval
$[-15,15]$, which covers the empirical range encountered in our experiments.
After $K$ steps, the estimated logit difference is converted to a probability
using the sigmoid function. A budget of $K$ binary-search steps therefore
requires $K$ model calls and $K$ logit interventions per example.

Table~\ref{tab:query_cost_baselines} summarizes the resulting query costs.

\begin{table}[h]
\centering
\caption{Query cost per sample for the evaluated restricted-access methods.}
\label{tab:query_cost_baselines}
\begin{tabular}{lccc}
\toprule
\textbf{Estimator} & \textbf{Model Calls} &
\textbf{Generated Tokens / Call} & \textbf{Logit Interventions} \\
\midrule
\textbf{One-bit (ours)}      & \textbf{1} & \textbf{1} & \textbf{1} \\
Verbalized Confidence        & 1   & 10 & 0 \\
Sampling Frequency ($K$)     & $K$ & 1  & 0 \\
Carlini Binary Search ($K$)  & $K$ & 1  & $K$ \\
\bottomrule
\end{tabular}
\end{table}

Our one-bit procedure uses exactly one threshold query per example, whereas
Sampling Frequency and Carlini-style score reconstruction require repeated
model calls as the query budget $K$ increases.

\newpage
\section{Additional LLM Results}
\label{app:llm_results}

\subsection{Calibration Auditing}
\label{app:calib_llm_xp}

We report the full LLM calibration-auditing results behind
Figure~\ref{fig:calibration_main}(c)--(d). We show the query--error trade-off
for each model--dataset pair and the corresponding one-bit calibration
profiles.

\paragraph{Per-model query--error trade-offs.}
Figure~\ref{fig:calibration_cost_all} shows the calibration-estimation error
for all four models on BoolQ and MMLU. For each method, we measure the mean
squared error relative to the probability-access
$\widehat{\mathrm{ECE}}_M$ computed on the same evaluation set. The one-bit
estimator and Verbalized Confidence use one model call per example, while
Sampling Frequency and Carlini-style logit extraction use
$K\in\{1,\ldots,8\}$ calls.

\begin{figure}[h]
    \centering
    \includegraphics[width=0.98\textwidth]
    {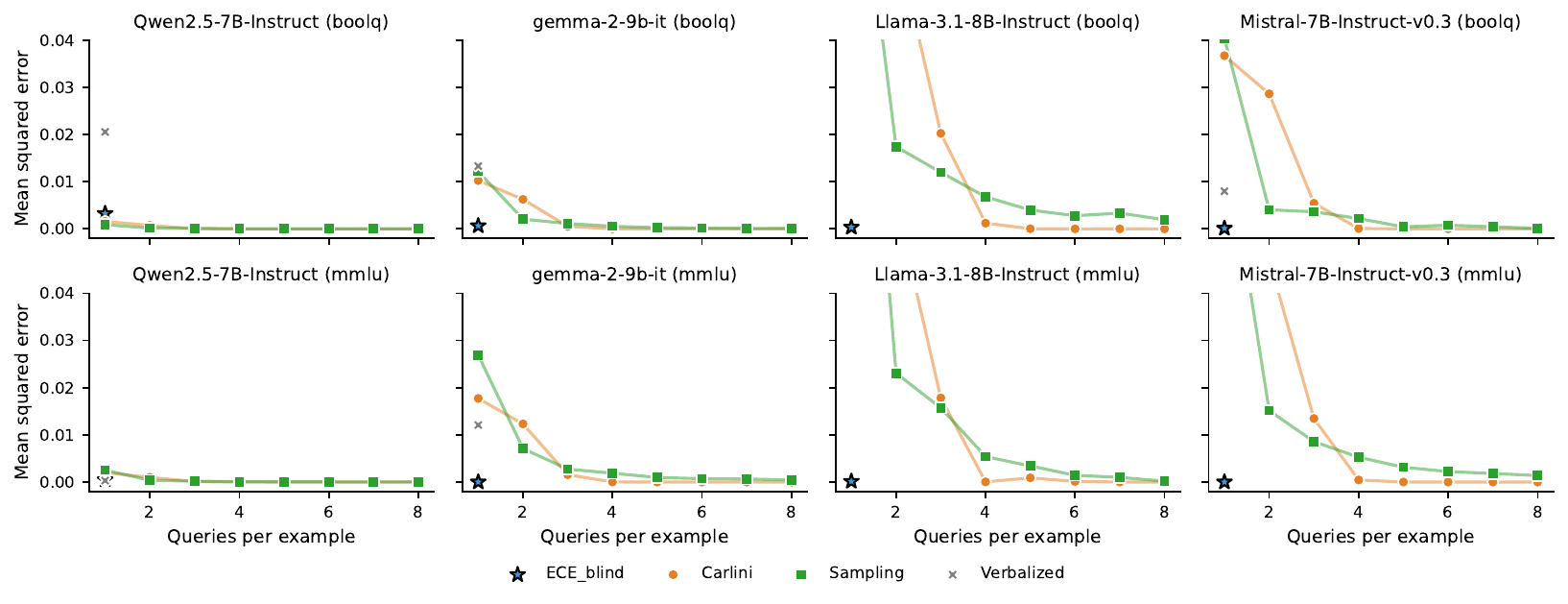}
    \caption{
    \textbf{Calibration-estimation error across model--dataset pairs.}
    Mean squared error relative to the probability-access ECE reference as a
    function of the number of queries per example. Results are shown for the
    four evaluated LLMs on BoolQ (top) and MMLU (bottom). The one-bit estimator
    uses one threshold query, while Sampling Frequency and Carlini-style logit
    extraction use repeated queries. The vertical range is restricted for
    readability; numerical results are reported in
    Table~\ref{tab:calibration_auditing_full}.
    }
    \label{fig:calibration_cost_all}
\end{figure}

At one query per example, the one-bit estimator has the lowest MSE in six of
the eight model--dataset pairs. The two exceptions both involve Qwen: on
BoolQ, one-sample Sampling Frequency has lower error, while on MMLU,
Verbalized Confidence is slightly more accurate. Qwen is atypical among the
evaluated models because its predicted probabilities are strongly concentrated
near $0$ and $1$. In this regime, even coarse estimates such as a single
sample can approximate the hidden scores relatively well, which partly
explains the unusually strong Qwen baselines. Sampling improves as the query
budget increases, while Carlini-style extraction becomes very accurate after
several queries as it progressively reconstructs the hidden score. Overall,
the one-bit estimator performs best in the low-query regime across most
model--dataset pairs.

\paragraph{Numerical results.}
Table~\ref{tab:calibration_auditing_full} reports the query--error comparison
at representative budgets. The probability-access ECE is shown as a reference.
For easier interpretation, we report root mean squared error (RMSE), obtained
by taking the square root of the MSE used in
Figure~\ref{fig:calibration_cost_all}. RMSE is therefore on the same scale as
ECE.

\begin{table}[t]
\centering
\caption{
\textbf{Calibration-auditing results across all model--dataset pairs.}
The third column reports the probability-access ECE reference.
All remaining entries report root mean squared error (RMSE) relative to this
reference. One-bit and Verbalized Confidence use one query per example;
Sampling and Carlini-style extraction are reported for
$K\in\{1,4,8\}$ queries. Lower RMSE is better.
}
\label{tab:calibration_auditing_full}
\setlength{\tabcolsep}{3pt}
\resizebox{\textwidth}{!}{%
\begin{tabular}{llc cc ccc ccc}
\toprule
&&&
\multicolumn{2}{c}{\textbf{Single query}}
& \multicolumn{3}{c}{\textbf{Sampling}}
& \multicolumn{3}{c}{\textbf{Carlini}} \\
\cmidrule(lr){4-5}
\cmidrule(lr){6-8}
\cmidrule(lr){9-11}
\textbf{Model}
& \textbf{Dataset}
& \textbf{ECE ref.}
& \textbf{One-bit}
& \textbf{Verb.}
& $K=1$
& $K=4$
& $K=8$
& $K=1$
& $K=4$
& $K=8$ \\
\midrule

Gemma-2-9B
& BoolQ
& 0.218
& 0.025
& 0.115
& 0.110
& 0.024
& 0.010
& 0.101
& 0.003
& 0.002 \\

Llama-3.1-8B
& BoolQ
& 0.114
& 0.018
& 0.208
& 0.305
& 0.082
& 0.044
& 0.271
& 0.035
& 0.001 \\

Mistral-7B-v0.3
& BoolQ
& 0.265
& 0.011
& 0.089
& 0.201
& 0.047
& 0.011
& 0.192
& 0.009
& 0.001 \\

Qwen2.5-7B
& BoolQ
& 0.346
& 0.057
& 0.143
& 0.030
& 0.004
& 0.002
& 0.040
& $<0.001$
& $<0.001$ \\

\midrule

Gemma-2-9B
& MMLU
& 0.131
& 0.001
& 0.110
& 0.164
& 0.043
& 0.021
& 0.133
& 0.004
& $<0.001$ \\

Llama-3.1-8B
& MMLU
& 0.060
& 0.012
& 0.404
& 0.371
& 0.073
& 0.014
& 0.266
& 0.006
& 0.001 \\

Mistral-7B-v0.3
& MMLU
& 0.145
& 0.001
& 0.298
& 0.274
& 0.073
& 0.037
& 0.240
& 0.021
& $<0.001$ \\

Qwen2.5-7B
& MMLU
& 0.194
& 0.021
& 0.018
& 0.050
& 0.009
& 0.004
& 0.045
& 0.001
& $<0.001$ \\

\bottomrule
\end{tabular}%
}
\end{table}

\paragraph{Calibration profiles.}
Figure~\ref{fig:calibration_profiles_all} extends the calibration profile shown
in Figure~\ref{fig:calibration_main}(d) to every model--dataset pair. For each
score interval, we plot the signed one-bit contribution
$\widehat{\Delta}^{\mathrm{1bit}}_m$ together with the corresponding
probability-access contribution. Positive values indicate underconfidence and
negative values indicate overconfidence.

\begin{figure}[b]
    \centering
    \includegraphics[width=0.98\textwidth]
    {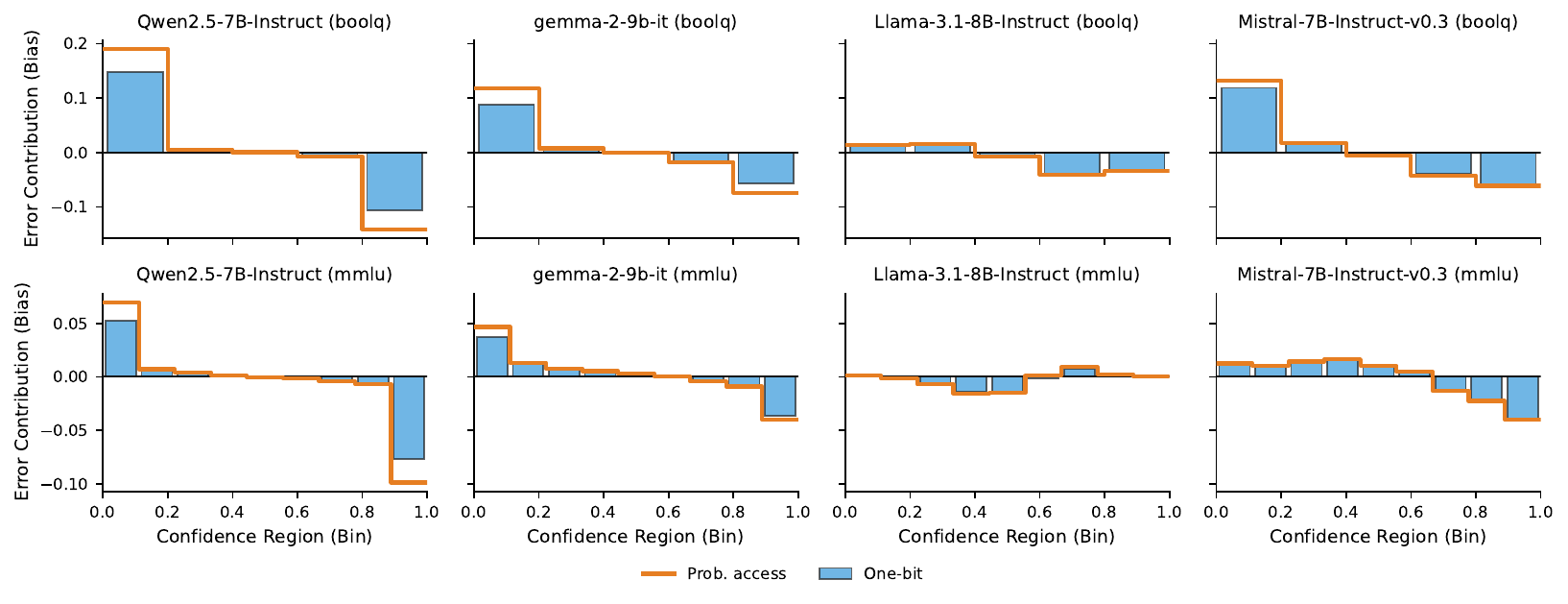}
    \caption{
    \textbf{One-bit calibration profiles across all model--dataset pairs.}
    Signed calibration-error contributions for the four evaluated LLMs on
    BoolQ (top) and MMLU (bottom). Bars show the one-bit estimates and curves
    show the corresponding probability-access contributions. Positive values
    indicate underconfidence and negative values indicate overconfidence. The
    one-bit profiles generally recover the main structure and sign of the
    probability-access profiles.
    }
    \label{fig:calibration_profiles_all}
\end{figure}

Across the eight settings, the one-bit profiles generally reproduce the main
regions of positive and negative calibration error without observing the
continuous scores. The agreement is close for several model--dataset pairs,
while larger differences appear in some score regions. These profiles
complement the aggregate $\widehat{\mathrm{CE}}^{\mathrm{1bit}}_M$ by showing
where the estimated miscalibration comes from.

\clearpage
\subsection{Proper Scoring Rules}
\label{app:psr_llm}

We also evaluate one-bit proper-score estimation on the four LLMs and two
datasets used in the main experiments. The hidden scores are used only to
compute the probability-access reference and to emulate the threshold response
$B=\mathbf{1}\{S\geq T\}$. The one-bit estimator itself observes only
$(T,B,Y)$ and uses one threshold per example.

For each model--dataset pair, we repeat the threshold sampling $R=500$ times.
For the Brier score, we use Uniform thresholds. For Log-loss, we compare
Uniform thresholds with the fixed $\mathrm{Beta}(1/2,1/2)$ proposal. We report
RMSE relative to the probability-access risk.

\begin{figure*}[h]
    \centering
    \includegraphics[width=0.7\textwidth]
    {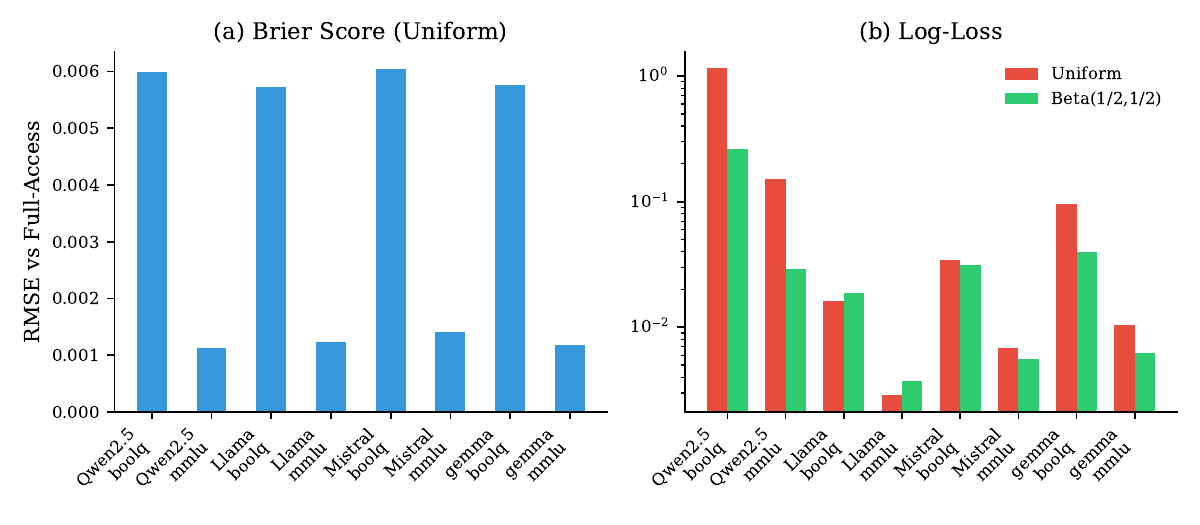}
    \caption{
    \textbf{One-bit proper-score estimation on LLMs.}
    RMSE relative to the probability-access risk for all eight model--dataset
    pairs, averaged over $R=500$ threshold draws.
    \textbf{(a)} Brier score with Uniform thresholds.
    \textbf{(b)} Log-loss with Uniform and
    $\mathrm{Beta}(1/2,1/2)$ thresholds; the vertical axis is logarithmic.
    Each one-bit estimate uses one threshold query per example.
    }
    \label{fig:psr_llm}
\end{figure*}

Figure~\ref{fig:psr_llm}(a) shows that the one-bit Brier estimator is accurate
across all model--dataset pairs, with RMSE below $0.007$ in every case.

For Log-loss, the threshold proposal matters more. As shown in
Figure~\ref{fig:psr_llm}(b), $\mathrm{Beta}(1/2,1/2)$ gives large gains for
Qwen, reducing RMSE from about $1.15$ to $0.26$ on BoolQ and from $0.15$ to
$0.029$ on MMLU. It also improves both Gemma settings and both Mistral
settings, while Uniform sampling is slightly better for Llama. These results
support $\mathrm{Beta}(1/2,1/2)$ as a simple endpoint-focused fixed proposal,
while also showing that no single proposal is best in every setting.

\subsection{Temperature Scaling}
\label{app:ts_llms}

We next report the individual model--dataset results behind the averaged LLM
Temperature Scaling experiment in Section~\ref{sec:exp_temperature_scaling}.
We use the same $100$ random calibration/test splits, with
$N_{\mathrm{cal}}=1024$ original questions per split. All methods use the same
split within each repetition, and MMLU examples derived from the same original
question are kept together.

Figure~\ref{fig:ts_llm_all_pairs} reports the fraction of the
probability-access TS improvement recovered by each method. Probability-access
TS is therefore fixed at $100\%$. The two one-bit methods use a single
threshold query per calibration example, while Carlini-style extraction uses
$K\in\{1,\ldots,8\}$ repeated queries.

\begin{figure*}[t]
    \centering
    \includegraphics[width=\textwidth]
    {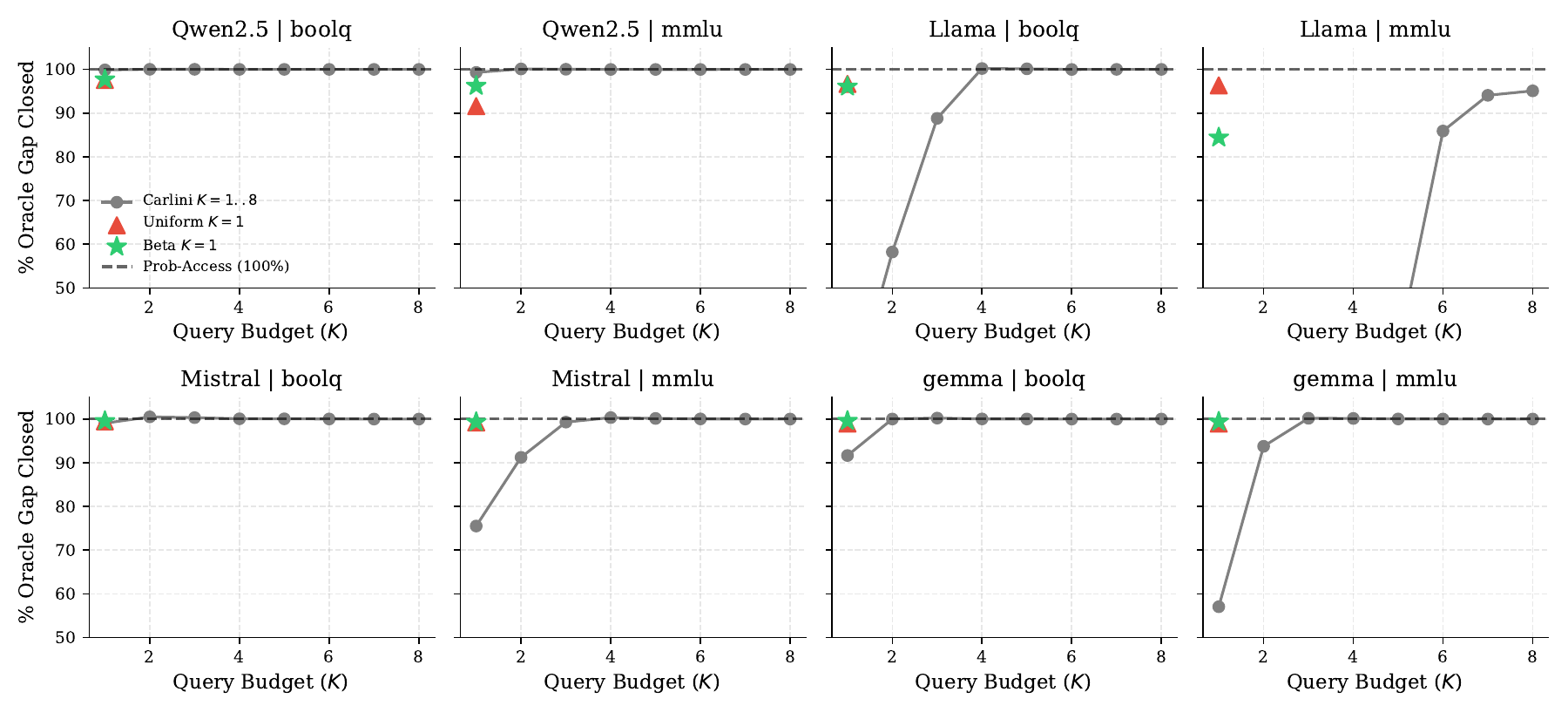}
    \caption{
    \textbf{Temperature Scaling on individual LLM--dataset pairs.}
    Percentage of the probability-access TS improvement recovered on the
    held-out set. One-bit Uniform and
    $\mathrm{Beta}(1/2,1/2)$ use one threshold query per calibration example;
    Carlini-style extraction uses $K$ repeated queries.
    The dashed line denotes probability-access TS.
    Values below $50\%$ fall outside the displayed range.
    }
    \label{fig:ts_llm_all_pairs}
\end{figure*}

With one query per example, centered one-bit TS remains close to
probability-access TS across most model--dataset pairs. Carlini-style
extraction improves as the query budget grows, but the number of queries
needed varies substantially across models.

\paragraph{Optimization details.}
Our LLM Temperature Scaling code is adapted from the
\texttt{probmetrics} implementation~\citep{berta2025rethinking}.
We use PyTorch \texttt{LBFGS} with learning rate $0.1$, at most $20$
iterations, and the default tolerances
$\texttt{tolerance\_grad}=10^{-5}$ and
$\texttt{tolerance\_change}=10^{-9}$.
We enforce $\alpha\geq10^{-3}$ by clamping the parameter inside the
optimization closure.

% \section{Reproducibility Details}
% \label{app:reproducibility}
% % Random seeds
% % Software / library versions
% % Hardware
% % Code availability
\end{document}